\documentclass[]{jair}

\setcopyright{cc}
\copyrightyear{2025}
\acmDOI{10.1613/jair.1.18402}

\JAIRAE{Martin Gebser}
\JAIRTrack{Awards Track} 
\acmVolume{84}
\acmArticle{32}
\acmMonth{12}
\acmYear{2025}

\RequirePackage[
  datamodel=acmdatamodel,
  style=acmauthoryear,
  backend=biber,
  giveninits=true
  ]{biblatex}

\usepackage{soul}
\usepackage{cancel}
\usepackage{url}
\usepackage[utf8]{inputenc}
\usepackage{caption}
\usepackage{graphicx}
\usepackage{amsmath}
\usepackage{amsthm}
\usepackage{booktabs}
\usepackage{algorithm}
\usepackage{algorithmic}

\newtheorem{example}{Example}
\newtheorem{theorem}{Theorem}
\newtheorem{definition}{Definition}
\newtheorem{proposition}{Proposition}
\newtheorem{lemma}{Lemma}
\newtheorem{remark}{Remark}
\newtheorem{corollary}{Corollary}

\usepackage{soul} 
\usepackage[normalem]{ulem} 
\usepackage{graphicx}
\usepackage{color}
\usepackage{braket} 

\usepackage{tabularx}

\usepackage{todonotes} 

\usepackage{tikz}
\usetikzlibrary{arrows}
\usetikzlibrary{positioning}
\usetikzlibrary{decorations.pathmorphing}

\newcommand{\bm}[1]{\underline{\smash{#1}}}
\newcommand{\bmnested}[1]{\underline{#1\vphantom{p}}}
\newcommand{\fquo}[1]{\bm{\ensuremath{#1}}}
\newcommand{\quo}[1]{\bm{\textit{#1}}}
\newcommand{\bland}{\bm \land}
\newcommand{\blor}{\bm \lor}
\newcommand{\bforall}{\bm \forall}
\newcommand{\bneg}{\bm \neg}

\newcommand{\quot}{\ensuremath{\mathbb{Q}}} 
\newcommand{\sigbase}{\ensuremath{S_b}}
\newcommand{\ist}{\ensuremath{\textit{ist}}}
\DeclareRobustCommand{\truth}{\ifmmode \mathbb{T} \else $\mathbb{T}$ \fi} 

\newcommand{\hts}{H_\truth}
\newcommand{\hqf}{H_{\textit{ist}}^\textit{fin}}
\newcommand{\htf}{H_{\truth}^\textit{fin}}
\newcommand{\hof}{H_{\textit{tools}}^\textit{fin}} 

\newcommand{\hq}{H_C}
\newcommand{\hf}{H_C^\textit{fin}}
\newcommand{\htr}{H_{\truth}}
\newcommand{\hfoinf}{H_{\bforall}}
\newcommand{\hcinf}{H_C^\infty}

\newcommand{\wff}{\text{Wff}}
\newcommand{\wft}{\textit{Wft}}

\newcommand{\sub}{\textit{Sub}}
\newcommand{\reach}{\textit{Term}} 
\newcommand{\E}{\text{E}}
\newcommand{\uq}{\mu^{\!\scalebox{0.75}[1.0]{-}\!\!\:1}}

\newcommand{\fq}{\ensuremath{\bm{F}}}
\newcommand{\fpq}{\ensuremath{\bm{P}}}

\newcommand{\cv}{\ensuremath{\bm{V}}}

\newcommand{\formuset}{\ensuremath{\mathcal{L}} }

\newcommand{\termset}{\ensuremath{\mathcal{T}}}
\newcommand{\quotableFormu}{\ensuremath{\mathcal{L}_q}}
\newcommand{\quotableTerm}{\ensuremath{\mathcal{T}_q}}

\newcommand{\tq}{\ensuremath{\mathcal{Q}}}
\newcommand{\tqv}{\ensuremath{\mathcal{Q}_{v}}}
\newcommand{\tqt}{\ensuremath{\bm{\mathcal{T}}}}
\newcommand{\tqf}{\ensuremath{\bm{\mathcal{L}}}}
\newcommand{\tqtv}{\ensuremath{\bm{\mathcal{T}_v}}}
\newcommand{\tqfv}{\ensuremath{\bm{\mathcal{L}_v}}}

\newcommand{\context}[1]{\ensuremath{#1^\ist\ }}
\newcommand{\contextq}[1]{\ensuremath{#1^{\ist q}\ }}

\newcommand{\types}{\ensuremath{U}}

\newcommand{\tauto}{\ensuremath{\textit{Tauto}}}

\newtheorem*{theorem*}{Theorem}
\newtheorem*{proposition*}{Proposition}

\newcounter{qianaax}

\newcommand{\tagqianaaxiom}{
\addtocounter{qianaax}{1}
\tag{A\arabic{qianaax}}}

\newcounter{tautodef}

\newcommand{\tagtautodef}{
\addtocounter{tautodef}{1}
\tag{T\arabic{tautodef}}}

\newcommand{\expofinite}{\ensuremath{^{\text{fin}}}}

\newcommand{\textiff}{\text{ iff }}

\newcommand{\assignment}{\sigma}

\newcommand{\ignore}[1]{}

\setuldepth{Berlin}

\usepackage{xcolor}
\newcommand\SCI[1]{{\color{violet}{\it \bf Simon :} #1}}
\newcommand\FMSI[1]{{\color{blue}{\it \bf Fabian :} #1}}
\newcommand\PHPI[1]{{\color{green}{\it \bf PH: } #1}}
\newcommand{\fms}[1]{\textcolor{magenta}{Fabian: #1}}
\newcommand\SC[1]{{\todo[color=violet!40, inline]{Simon: #1}}}
\newcommand\FMS[1]{{\todo[color=blue!40, inline]{Fabian: #1}}}
\newcommand\ph[1]{{\todo[color=green!40, inline]{PH: #1}}}
\newcommand{\FS}[1]{\todo[inline]{François: #1}}
\renewcommand\SC[1]{}
\renewcommand\FMS[1]{}
\renewcommand\ph[1]{}
\renewcommand{\fms}[1]{}
\renewcommand\SCI[1]{}
\renewcommand\FMSI[1]{}
\renewcommand\PHPI[1]{}

\usepackage{blindtext}
\usepackage{subfiles} 
\graphicspath{{images/}{../../images/}}

\usepackage{enumitem}
\usepackage{nameref}
\usepackage{mdframed} 
\newmdenv[
  topline=false,
  bottomline=false,
  rightline=false,
  linecolor=lightgray,
  linewidth=4pt,
  innertopmargin=1pt,
  innerbottommargin=1pt,
  innerleftmargin=5pt,
  innerrightmargin=0pt,
  skipabove=5pt,
  skipbelow=5pt
]{siderules}

\newcommand{\myexample}[1]{\begin{siderules} #1\end{siderules}} 

\begin{document}
\title[Qiana]{Qiana: A First-Order Formalism to Quantify over Contexts and Formulas with Temporality} 


\author{Simon Coumes}
\orcid{0009-0005-8888-6495}
\authornote{Corresponding Author.}
\email{simon.coumes@telecom-paris.fr}
\affiliation{%
  \institution{Telecom Paris, Institut Polytechnique de Paris}
  \city{Palaiseau}
  \state{Île-de-France}
  \country{France}
}

\author{Pierre-Henri Paris}
\orcid{0000-0002-9665-1187}
\email{pierre-henri.paris@universite-paris-saclay.fr}
\affiliation{%
  \institution{Université Paris-Saclay}
  \city{Gif-sur-Yvette}
  \state{Île-de-France}
  \country{France}
}

\author{François Schwarzentruber}
\orcid{0000-0002-1228-4333}
\email{francois.schwarzentruber@ens-lyon.fr}
\affiliation{%
  \institution{ENS Lyon}
  \city{Lyon}
  \country{France}
}

\author{Fabian Suchanek}
\orcid{0000-0001-7189-2796}
\email{fabian.suchanek@telecom-paris.fr}
\affiliation{%
  \institution{Telecom Paris, Institut Polytechnique de Paris}
  \city{Palaiseau}
  \state{Île-de-France}
  \country{France}
}

\renewcommand{\shortauthors}{Coumes, Paris, Schwarzentruber \& Suchanek}

\begin{abstract}
We introduce Qiana, a logic framework for reasoning on formulas that are true only in specific contexts. 
In Qiana, it is possible to quantify over both formulas and contexts to express, e.g., that 
``everyone knows everything Alice says''.  
Qiana also permits paraconsistent logics within contexts, so that contexts can contain contradictions. 
Furthermore, Qiana is based on first-order logic, and is finitely axiomatizable, so that Qiana theories are compatible with pre-existing first-order logic theorem provers. We show how Qiana can be used to represent temporality, event calculus, and modal logic. We also discuss different design alternatives of Qiana. 
\end{abstract}

\maketitle

\section{Introduction} \label{sec:Introduction}
    In his ``Notes on formalizing contexts''~\cite{McCarthy87Turing}, John McCarthy argued for the importance of context representation in formal logic. 
The core idea is that statements can be tied to specific contexts, which act as modalities on the statements.
This idea is substantiated by the predicate \emph{ist}: 
In McCarthy's notations, $\textit{ist}(c,\varphi)$ means that the formula $\varphi$ is true in the context~$c$. 
Contexts can represent different things: 
Something can be true only in the context of a newspaper article, in the context of a piece of fiction, or in someone's beliefs.
We illustrate one possible use of contexts with the final scene of the play ``Romeo and Juliet'' by William Shakespeare:

\myexample{
\textit{Near the end of the play, Juliet wishes to meet with Romeo, but her parents won't let her. Her friend, Friar Laurence, offers her a potion and says it will allow her to fake her death. 
Juliet takes the potion, hoping it will allow her to escape her family.
However, the plan backfires: Romeo sees Juliet before she awakens, seemingly dead, and kills himself in despair. When Juliet later wakes up, she sees Romeo dead and kills herself. }
}

The key elements of the ending of the play are: (1) Friar Laurence is right in what he says (the potion will make Juliet appear dead), and (2) someone who is madly in love with someone else will kill themselves if they believe their loved one to be dead. Thus, leaving out details and overgeneralizing, we want to represent:
\begin{align*} 
    & ~~~~\forall \phi.\ \textit{ist(says(FriarLaurence)},\phi) \rightarrow \phi\\
    & ~~~~\forall x,y.\ \textit{madlyLoves}(x,y) \wedge \textit{ist(believes}(x), \textit{dead}(y)) \rightarrow \textit{willSuicide}(x)
\end{align*}
\noindent Here
\textit{believes}$(x)$ is the context of the beliefs of the agent $x$. Our example leads us to the following desiderata for expressivity:
\begin{description}
    \item[1. Truth Representation:] the ability to link truth in reality and in contexts (``$ist(c,\varphi)\!\!\rightarrow\!\!\varphi$'') 
    \item[2. Formula Quantification:] the ability to quantify over formulas (``$\forall \varphi.\ ist(c, \varphi)$'')
    \item[3. Context Quantification:] the ability to quantify over contexts (``$\forall c.\ ist(c, \varphi)$''), or even certain forms of context (``$\forall x.\ \textit{ist}(\textit{believes}(x), \varphi)$'')
\end{description}
\noindent Moreover, we want to perform automated reasoning, or at least semi-automated reasoning:
\begin{description}
    \item[4. Semi-decidability: ] Logical entailment $\Gamma \models \phi$ should be semi-decidable. 
\end{description}
 
%

\noindent Fulfilling these desiderata simultaneously is not trivial. One difficulty is that Desideratum 1 invites complications from the Theorem of Undefinability of Truth of \cite{Tarski1936-TARTCO}: A language cannot fully describe its own truth, assuming it includes basic arithmetic. 
This is because it allows self-referential statements, which lead to contradictions.

One way to do contextual reasoning in logic is through modal logic. However, modal logic does not consider formulas as objects that one can quantify over.
Another classical way would be to use higher-order logics, but these (typically) quantify over predicates rather than the syntactic formulas themselves. Furthermore, they are usually not even semi-decidable. 
\cite{MOORE1980} proposes to \emph{quote} formulas as terms within the logic.
However, the notion of context in this approach is very restrictive.
For example, it lacks a dedicated mechanism to express statements as simple as
``\textit{If Juliet believes all Capulets are nice, then for any Capulet $x$, she believes $x$ is nice}''.

It seems that the promising idea of using object-level counterparts to formulas within first-order logic was never explored to produce a suitable framework for this form of general contextual reasoning.
Thus, to the best of our knowledge, no logical framework currently satisfies all 4 desiderata simultaneously (see Table~\ref{fig:tableDesi}, discussed in the related work section).

This article is an extended version of our previous conference paper~\cite{QianaKR24}, which proposes representing formulas within contexts as regular terms of the logic that obey a specific axiomatization. 
We borrow the \textit{ist} predicate from \cite{McCarthy93Context}.
We follow the idea of \cite{MOORE1980} to build terms that are structurally similar to formulas. 
(This idea is itself 
an extension of G\"odel's numbers, see~\cite{godel1931formal}.)
We use the idea of \cite{Tarski1936-TARTCO} 
to introduce a special truth predicate and ensure that this predicate cannot be quoted. 
We then show how these components can be axiomatized so that Desiderata 1-4 are fulfilled without falling for the complications of Tarski's theorem. 
The resulting framework, Qiana (Quantifying over Agents and Assertions), is finitely axiomatizable and can thus be used with any First-Order-Logic theorem prover.
We also introduce a special character $\quot$ to nest quotations within quotations. 
This allows for a larger array of manipulations around contexts, which are notably useful for our finite axiomatization process. Qiana can model agents' beliefs (as in our Romeo and Juliet example). It can also be used for paraconsistent reasoning (where a context contains contradictory statements), or to describe the differences between two fictional contexts (e.g., two versions of the same story). 

This paper 
first follows our original paper on Qiana~\cite{QianaKR24}: Section~\ref{sec:RelatedWork} discusses the related work; Section~\ref{sec:Preliminaries} introduced notations;
Section~\ref{sec:FormalismDef} explains how Qiana quotes formulas; 
Section~\ref{sec:def_qiana} defines Qiana;
Section~\ref{sec:Example} discussed simple applications of Qiana; and
Section~\ref{sec:FiniteAxio} describes the finite axiomatization process of Qiana for use with automated theorem provers.

In the second part of this paper, we provide more discussion that goes beyond the original paper~\cite{QianaKR24}: 
In Section~\ref{sec:Time}, we extend Qiana to reason about temporality and events. In Section~\ref{sec:Types}, we present an alternative version of Qiana that is based on typed logic. 
 In Section~\ref{sec:OtherFormalisms}, we discuss how usual modal logics can be represented in Qiana.
Finally, Section~\ref{sec:Conclusion} concludes.
Supplementary material, including the proofs of our theorems and the code of our implementation, is available at \url{https://github.com/dig-team/Qiana}.
    
\section{Related Work} \label{sec:RelatedWork}

\newcommand{\signYes}{yes}
\newcommand{\signNo}{no}
\newcommand{\signDunno}{?}
\newcommand{\signNA}{NA}
\begin{table*}[ht]
    \centering    
    \caption{Semi-decidability and the desiderata of Truth Representation, Formula Quantification, \\ \hspace*{1.2cm} and Context Quantification} \label{fig:tableDesi}
    \begin{tabular}{l l l l l} 
     \toprule
      & Truth$~~~$ & Formula & Context & Semi-\\ 
      & Representation & Quantif & Quantif &decidable \\
     \midrule
     \cite{MOORE1980} & \signYes & \signYes & \signNA & \signYes\\      
     \cite{kif} & \signYes & \signYes & \signYes & \signNo \\ 
     \cite{MultiModalHalpernM92} & \signYes & \signNo & \signNo & \signYes \\  
     \cite{McCarthy93Context} & \signYes & \signNo & \signNo & \signNA \\ 
     \cite{giunchiglia1993contextual} & \signNA & \signNo & \signNo & \signNA \\ 
     \cite{BuvacM93Propositional}& \signNA & \signNo & \signNo & \signYes \\ 
     \cite{Buvac96FOL} & \signNA & \signNo & \signYes & \signNo \\ 
     \cite{DBLP:journals/ai/GhidiniG01} & \signYes & \signNo & \signNo & \signYes \\
     \cite{DBLP:conf/sbia/Perrussel02} & \signNo & \signNo & \signYes & \signNo \\    
     \cite{DBLP:journals/puc/RanganathanC03} & \signNA & \signNo & \signNo & \signYes\\
     \cite{carroll2005named} & \signYes & \signNo & \signNo & \signYes\\
     \cite{DBLP:conf/ijcai/BrewkaEFW11} & \signNo & \signNo & \signNo & \signYes\\
     \cite{commonlogic} & \signYes & \signNo & \signYes & \signNA \\ 
     \cite{DBLP:conf/iccs/AljalboutBF19} & \signYes & \signNo & \signNo & \signYes \\    
     \cite{higher-order}$~~~$ & \signNA & \signNo & \signNA & \signNo \\ 
     \cite{rdf-star} & \signYes & \signNo & \signNo & \signYes\\     
     Qiana & \signYes & \signYes & \signYes & \signYes\\   
     \bottomrule
    \end{tabular}
    \end{table*}


John McCarthy observed that many statements are true only in a specific context~\cite{McCarthy87Turing}.
Several follow-up works have elaborated on this idea, but none of them allow for Context Quantification and Formula Quantification. 
The first of these elaborations was by McCarthy himself \cite{McCarthy93Context}. He proposed to write 
$\textit{ist}(c,p)$ to say that $p$ is a proposition that is true in the context $c$. Thus, contexts are treated as objects representing a state of the universe at a given instant. However, this work was based on propositional logic. Hence, it cannot deal with first-order formulas, let alone quantify over contexts or formulas. 

\cite{BuvacM93Propositional} and \cite{BuvacBM94Propositional} formalized 
a propositional modal logic version of McCarthy's idea, which is sound, complete, and decidable. 
In their formalism, \emph{ist} is treated as a binary modality over propositions. 
Again, there is no possibility of quantifying over contexts or formulas.
\cite{Buvac96FOL} extended this work to first-order logic and allowed the description of contexts through properties. For instance, $\forall c.\ p(c)\rightarrow \textit{ist}(c,\phi)$ means that 
the formula $\phi$ is true in every context with the property $p$. 
This logic is sound and complete; the work was the first to allow quantification over contexts. However, unlike our approach, all contexts must have perfect knowledge of each other's beliefs, i.e., everyone knows what everyone else thinks.
Furthermore, unlike our approach, \cite{Buvac96FOL} does not allow for quantification over formulas.\footnote{According to \cite{GuhaMF04ContextSemWeb}, it is also not semi-decidable. However, \cite{Buvac96FOL} contains proofs of completeness and soundness, which entail semi-decidability.}  

Moore's work on reasoning about knowledge~\cite{MooreTechReport,MOORE1980} avoids the issues of self-reference in higher-order logic 
by representing formulas as terms within the logic. A special truth predicate connects these terms to their formula counterparts. This will also be done in Qiana. However, Moore's notion of context 
is quite restrictive: its many-worlds semantics assumes that contexts are logically omniscient (if something is true within a context, then all its consequences are also true). 
This is unsuitable to represent the knowledge of humans, whose reasoning depth is limited.
It also lacks an equivalent to our special escape function symbol $\quot$, which is used to put any given value into a quotation. 
Such a feature is important 
to present axioms that connect what is true outside of contexts to what is true within them,
e.g., for
statements of the form
``If in a context it is true that a statement holds for all $x$, then that statement holds for all $x$ in that context''.
Furthermore, there is no finite axiomatization, and thus, the method does not allow the use of state-of-the-art theorem provers that Qiana permits. 

Other works fall in the realm of epistemic and doxastic logics, which deal with the knowledge and beliefs of agents, respectively. The modal approach has been widely adopted for both cases~\cite{epistemic}.
For instance, \cite{MultiModalHalpernM92} proposed 
a multi-modal logic to deal with the knowledge and beliefs of multiple agents,
where each agent has its own operators.
For example, $\mathsf{K}_i \phi$ means that the agent $i$ knows $\phi$, and $\mathsf{B}_i \phi$ means  
that $i$ believes $\phi$.
Considering only the knowledge operators, 
this logic is equivalent to the formalism of \cite{BuvacM93Propositional}, 
where each context is equivalent to a specific modality. 
However, these modal approaches have no way to quantify over contexts. 
Furthermore, these approaches focus on propositional logic and cannot deal with first-order formulas like Qiana. 

\cite{giunchiglia1993contextual} and \cite{giunchiglia1997introduction} treat each context as a logical theory with its own language, set of axioms, and set of rules.
The main goal in this series of works is the translation of formulas from one context to another. 
The works in this series study only the propositional case and introduce no quantification.
\cite{DBLP:journals/ai/GhidiniG01} 
propose that contexts need two principles: locality (what is known by the agent) and compatibility (enforcing a kind of coherence in viewpoints). While this approach can deal with first-order formulas, it does not allow 
for quantified formulas or quantified contexts. 

The Knowledge Interchange Format KIF~\cite{kif} is a data format for database knowledge exchange. 
With the help of a quotation operator, a formula can be reified and handled as a syntactic element. 
However, KIF does not admit any complete proof theory. It is not even semi-decidable because it goes beyond first-order logic~\cite{higher-order}. 
KIF's successor, Common Logic~\cite{commonlogic} (CL), is a framework for a family of FOL-based languages. 
Unlike KIF, CL has no quotation operator, and it does not have a built-in mechanism for handling contexts. While some subsets of CL admit a complete proof theory, 
there is still no complete proof theory for Common Logic as a whole~\cite{commonlogic,proof_cl,menzel2013completeness}. 
\cite{FabianMT} proposes a translation from KIF to disjunctive logic programs, but also does not offer a complete proof theory.
\cite{DBLP:conf/sbia/Perrussel02} proposes a many-sorted modal first-order logic. This approach cannot 
quantify over formulas, or express formulas such as $\textit{ist}(c,\phi) \rightarrow \phi$ since no ``super context'' represents the real world. 

In the work of \cite{DBLP:journals/puc/RanganathanC03}, contexts are first-order predicates like \emph{Location} or \emph{Temperature}. Hence, the approach cannot deal with 
quantified formulas. 
\cite{DBLP:conf/ijcai/BrewkaEFW11} propose an approach to deal with different sources of knowledge. Each context is a knowledge base with its language, and bridge rules allow communication between contexts and handle inconsistencies. Intrinsically, it is not possible to quantify over formulas or contexts. 
More recently, \cite{DBLP:conf/iccs/AljalboutBF19} proposed a two-dimensional ontology language that allows defining context-dependent classes, properties, and axioms. It also allows expressing knowledge about contexts to reason on contextualized triples. However, it is impossible to quantify over formulas or contexts. Furthermore, the work uses description logics, which has limited expressiveness w.r.t. first-order logic.

Other approaches of the Semantic Web, like RDF-star~\cite{rdf-star} and named graphs~\cite{carroll2005named}, can handle context but not truth representation. They can handle neither quantification over contexts nor over formulas. 
One may think that second-order logic~\cite{higher-order} could be of help. However, classical higher-order logics allow quantification over predicates, not over formulas. 
\cite{SteenModal} extend the prover Leo-III with a form of higher-order modal logic, but still does not allow quantification over formulas. 

We thus conclude that no semi-decidable framework currently satisfies the desiderata of Truth Representation, Formula Quantification, and Context Quantification.


\section{Notations} \label{sec:Preliminaries}
    Our work relies on the usual notions of first-order logic (FOL) (see, e.g., ~\cite{BookLogic} for a primer).
We use standard syntactic sugar notations, 
writing, e.g., $\varphi \rightarrow \psi$ for~$\neg \varphi \lor \psi$.
We also use the usual substitution meta-notation:~$\varphi[x \leftarrow t]$ denotes the formula obtained by recursively replacing all occurrences of variable $x$ with $t$ in $\varphi$ until a quantification over $x$ is reached. 

For our purposes, a \emph{signature} $S$ is a tuple $(F,P,V_\infty,\delta)$, where $F$ is a set of function symbols, $P$ is a set of predicate symbols, $V_\infty$ is an infinite set of variables,
 and $\delta : P \cup F \rightarrow \mathbb{N}$ a function that gives the arity of each symbol. Constant symbols are function symbols of arity 0.
A given signature defines a set \termset of terms, and a set \formuset of formulas.

\newcommand{\sem}[1]{[\![#1 ]\!]}
A \emph{model} (sometimes called an \textit{interpretation} in the literature) is a tuple $(D, [\![ ]\!])$ where $D$ is a non-empty set called the \emph{domain of the model}, and $[\![ ]\!]$ is the interpretation mapping that maps each 
function symbol $f$ to a function $[\![f]\!] \in D^{\delta(f)}\!\rightarrow\!D$, and each predicate symbol $p$ to a function $[\![p]\!] \in D^{\delta(p)}\!\rightarrow\!\{0,1\}$.
An assignment is a partial function $\assignment : V_\infty \rightarrow D$.
Given an assignment for the free variables $\assignment$, we recall that terms (e.g., $1+x$) are interpreted as elements in the domain (e.g., $1+x$ is interpreted as the element $\sem{+}(\sem{1}, \assignment(x))$), and formulas (e.g., $p(x)$) are interpreted as true/false (e.g., the semantics of $p(x)$ is $\sem{p}(\assignment(x))$, which is either 0 or 1).

A theory is a set of formulas. Contrary to some definitions in the literature, we do not require theories to be closed under entailment.
An axiom schema is a formula with meta-variables (such as $\neg \neg \varphi \rightarrow \varphi$, where $\varphi$ is a meta-variable that stands for a formula). We will occasionally write the name of the axiom schema to stand for the set of all its instantiations in a given signature.

Let $M$ be a model, $\assignment$ an assignment of values in the domain of $M$ to free variables, and $H$ a theory. 
We write $M, \assignment \models \varphi$ to say that the formula $\varphi$ is true in model $M$, where all the free variables of $\varphi$ are defined in $\assignment$.
We omit $\assignment$ if there are no free variables.
We write $H \models \varphi$ to say that $\varphi$ is a semantic consequence of $H$.
A closed formula that is true in at least one model is coherent. A theory for which there is a model that makes all the formulas true is also called coherent.

\section{Quoting and Unquoting Formulas} \label{sec:FOdefs} \label{sec:FormalismDef} 
%

The main idea of Qiana is to represent formulas that are true only in a specific context by \emph{quoted formulas}. Technically, a quoted formula is a term that represents a formula. 
FOL does not allow to manipulate formulas as objects, which is why we need to introduce our quoted formulas, which are terms that serve as counterparts of formulas we can manipulate.
Intuitively speaking, quoting a formula consists of 
replacing each logical connective, variable, predicate, and function symbol with a fresh function symbol. We denote the quoted counterpart of a symbol $z$ by $\fquo{z}$:
\begin{example}
    The quotation of formula the $p(x) \land (1+x = 2)$ is the term
$ \fquo{\land}(\fquo{p}(\fquo{x}), \fquo{=}(\fquo{+}(\fquo{1},\fquo{x}),\fquo{2}))$ 
where $\fquo{p}, \fquo{\land},\fquo{1},\fquo{+},\fquo{x},\fquo{=},\fquo{2}$ are quoted counterparts to the original symbols $p$, $\land$, 1, +, x, =, and 2.
For convenience and readability, we can write the formula as
$\fquo{p}~(\fquo{x})~ \fquo{\land} ~ (\fquo{1}~\fquo{+}~\fquo{x}~\fquo{=}~\fquo{2})$ 
\end{example}

\noindent We need to quote formulas that already contain quotations. To this end, we introduce a special function symbol $\quot$ (read ``quote''), which acts as an escape character and provides a way to nest quotations. The symbol $\quot$ is called the \emph{escape operator} or the \emph{quote operator}.
The $\truth$ symbol is called truth. Its behavior is defined in Section~\ref{sec:axioms}.




\noindent To accommodate all these additional symbols, we extend our signature. We write $\sqcup$ for the disjoint union and define:
\ignore{
We furthermore assume that there is a function symbol $\ist$ of arity 2 in \sigbase, which we will need in Section~\ref{sec:FOdefs} to describe what is true in a given context.
Recall that the set $V_\infty$ of variables is infinite. 
For technical reasons (having a finite axiomatization, see Section~\ref{sec:FiniteAxio}), we only allow quoting variables from a fixed finite subset $V$ of $V_\infty$. 
Moreover, we assume that $|V|$ is ``sufficiently large''\footnote{ meaning it is greater than 2 and at least as large as the highest symbol arity of our logic \FS{expliquer un peu plus}}.
}
\begin{definition}[augmented signature] \label{def:qcs}
Given a signature \sigbase = $(F_b, P_b, V_\infty, \delta_b)$
and a finite $V \subseteq V_\infty$, 
the augmented signature $S$ is the tuple $(F,P,V_\infty,\delta)$  with
   \begin{itemize}
        \item $P = P_b \sqcup \{\truth\}$
        \item $F = F_b \sqcup \fq \sqcup \fpq \sqcup \cv \sqcup \{\bland, \bneg, \bforall, \quot\}$ where
        \begin{itemize}
            \item $\fq = \{\fquo{f} \mid f \in F_b\}$
            \item $\fpq = \{\fquo{p} \mid p \in P\}$
            \item $\cv = \{\fquo{x} \mid x \in V\}$
        \end{itemize}
        \item $V_\infty$ remaining the same
        \item $\delta$ specifying that $\bland$, $\bforall$, $\bneg$, and $\quot$ are of arities 2, 2, 1, and 1, respectively.
        Furthermore, $\fquo{f}$ has the same arity as $f$, and $\fquo{p}$ has the same arity as $p$. The arity of $\fquo{x}$ is 0 for all $x$. The arity of $\truth$ is 1.
   \end{itemize}
\end{definition}
\noindent Without loss of generality, we assume that all the new symbols we introduce are not already in $\sigbase$. In what follows, we assume a fixed signature $\sigbase$ with an augmentation $S$. This signature implicitly defines the set $\termset$ of all terms.
We write $a \blor b$ as syntactic sugar for $\bneg (\bneg a \bland ~ \bneg b)$, and $a \bm{\rightarrow} b$ as syntactic sugar for $\bneg(a) \blor b$.


\subsection{Quotation Sets}

Now that we have a quotation-compatible signature, we want to define the quotation function $\mu$, which takes a formula and returns its quoted equivalent. For this purpose, we have to introduce a number of subsets of the set $\termset$ of all terms, which can be
roughly described as follows:
\begin{itemize}[itemsep=0em,parsep=0em,topsep=0pt,partopsep=1ex]
\item $\tqt$ is the subset of all quotations of well-formed terms. For example, $\tqt$ contains the terms $\fquo{+}(\fquo{1}, \fquo{1})$ (which is the quotation of the term $+(1, 1)$) and $\fquo{f}(\quot(\fquo 1))$ (which is the quotation of the term $f(\fquo 1)$).
\item $\tqf$ is the subset of all quotations of well-formed (possibly not closed) formulas. For example, $\tqf$ contains the terms $\fquo{p}(\fquo x)$ (which is the quotation of the formula $p(x)$) and $\fquo{p}(\quot(\fquo 1))$ (which is the quotation of the formula $p(\fquo 1)$).
    \item $\tq$ is the set of \emph{all} terms made up of quotation symbols. The set $\tq$ includes both $\tqt$ and $\tqf$. However, it also contains quotations of non-well-formed terms and formulas. For example, $\tq$ contains the term $\fquo{p} ( \fquo{x}$ $\bland~\fquo 1 )$ (which is the ``quotation'' of the non-well-formed expression $p(x \land 1)$).
\end{itemize}

\noindent We provide Backus-Naur definitions of $\tqt$, $\tqf$, and $\tq$ below.
For each of the subsets $\tqt$, $\tqf$, and $\tq$ we introduce (respectively) $\tqtv$, $\tqfv$, and $\tqv$, which have similar definitions except that they also contain the construction $\quot(x)$ with $x$ a variable in $V$.
This is necessary to allow variables in quotations. 

We start by defining $\tq$ and $\tqv$ formally by two Backus-Naur Forms. Our grammars 
apply to 
quotations of both terms and formulas:

\begin{definition}
\label{def:QQv}
    The sets $\tq$ and $\tqv$ are defined inductively by: 
    \begin{align*}  
    & \tq \hspace*{0.1cm} := \fquo{x} \mid \bm f(t_1, \dots, t_n) \mid \bm p(t_1, \dots, t_n) \mid \bland (t_1, t_2) \mid \bneg (t) \mid \bforall (\bm x,t) \mid \quot(t) \\
    & ~~~~~~~~~\text{\emph{for}}~~ \fquo{x} \in \cv, \bm f \in \fq, \bm p \in \fpq, t, t_1, ..., t_n \in \tq \\
    & \tqv \hspace*{0.1cm} := \bm x \mid \bm f(t_1, \dots, t_n) \mid \bm p(t_1, \dots, t_n) \mid \bland (t_1, t_2) \mid \bneg (t_1) \mid \bforall (\bm x,t) \mid \quot(t) \mid \quot(x)\\
    & ~~~~~~~~~\text{\emph{for}}~~ \fquo{x} \in \cv, \bm f \in \fq, \bm p \in \fpq, t, t_1, ..., t_n \in \tqv
    \end{align*}
\end{definition}

\begin{remark}
    $\tqv$ is $\tq$ extended by the term $\quot(x)$ for each variable $x \in V$.
    Hence $\tq \subseteq \tqv$.
\end{remark}
\begin{example}
    Let $x$ be a variable and $f$ a function symbol of arity 1.
    Then $\bm f(\bm x) \in \tq$.
    Also, $\bm f(\quot(x)) \in \tqv, \not\in \tq$.
\end{example}



\noindent The subsets of $\tqt$, $\tqtv$, (resp. $\tqf$ and $\tqfv$) are also defined by Backus-Naur Forms; but this time, we permit only quotations of well-defined terms (resp. formulas) except with a special $\quot$ symbol. 
\begin{definition}
\label{def:TTvLLv} The sets $\tqt, \tqtv, \tqf, \tqfv$ are defined inductively by
    \begin{align*}  
        & \tqt \hspace*{0.1cm} := \bm x \mid \bm f(t_1, \dots, t_n)  \mid \quot(t_q)\\
        & ~~~~~~~~~\text{\emph{for}}~~ \fquo{x} \in \cv, \bm f \in \fq, t_q \in \tq, t_1, \dots, t_n \in \tqt \\
        & \tqtv \hspace*{0.1cm} := \bm x \mid \bm f(t_1, \dots, t_n)  \mid \quot(t_q) \mid \quot(x)\\
        & ~~~~~~~~~\text{\emph{for}}~~ \fquo{x} \in \cv, \bm f \in \fq, t_q \in \tq, t_1, \dots, t_n \in \tqtv \\
        & \tqf \hspace*{0.1cm} := \bm p(t_1, \dots, t_n) \mid \bm \varphi_1 \bland \bm \varphi_2 \mid \bneg \bm \varphi_1 \mid \bforall(\bm x, \varphi_1) \\
        & ~~~~~~~~~\text{\emph{for}}~~ \fquo{x} \in \cv, t_1, \dots, t_n \in \tqt, \bm p \in \fpq, \varphi_1, \varphi_2 \in \tqf \\
        & \tqfv \hspace*{0.1cm} := \bm p(t_1, \dots, t_n) \mid \bm \varphi_1 \bland \bm \varphi_2 \mid \bneg \bm \varphi_1 \mid \bforall(\bm x, \varphi_1) \\
        & ~~~~~~~~~\text{\emph{for}}~~ \fquo{x} \in \cv, t_1, \dots, t_n \in \tqtv, \bm p \in \fpq, \varphi_1, \varphi_2 \in \tqfv
    \end{align*}  
\end{definition}

\begin{remark}
    $\tqtv$ (resp $\tqfv$) is to $\tqt$ (resp. $\tqf$) what $\tqv$ is to $\tq$.
    They are identical in structure, except that $\tqtv, \tqfv$, and $\tqv$ have variables injected.
\end{remark}

\noindent In this definition, $\tqt$, $\tqtv$ (resp. $\tqf$, $\tqfv$) 
contain only quotations of well-formed terms (resp. formulas) -- except in $\quot(t_q)$ where $t_q$ may be a ``quotation'' of a non-well-formed expression.



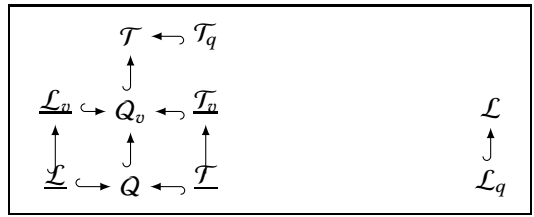
\begin{figure}[h]
    \centering
    \fbox{
    \tikzstyle{includedin} = [right hook-latex]
   \begin{tikzpicture}
    \node (termset) at (0, 2) {$\termset$};
    \node (quotaT) at (1, 2) {$\quotableTerm$};
    \node (tqv) at (0, 1)  {$\tqv$};
    \node (tq) at (0, 0)  {$\tq$};

    \node (tqfv) at (-1, 1)  {$\tqfv$};
    \node (tqf) at (-1, 0)  {$\tqf$};

    \node (tqtv) at (1, 1)  {$\tqtv$};
    \node (tqt) at (1, 0)  {$\tqt$};
  
   \draw[includedin] (tqfv) -- (tqv);
   \draw[includedin] (tqf) -- (tqfv);
\draw[-latex] (tqt) -- (tqtv);

   \draw[includedin] (tqt) -- (tq);
   \draw[includedin] (tqtv) -- (tqv);
   
   \draw[includedin] (tqf) -- (tq);
   \draw[includedin] (tq) -- (tqv);
   \draw[includedin] (tqv) -- (termset);
   \draw[includedin] (quotaT) -- (termset);
   \end{tikzpicture}
   \hspace{3cm}
   \begin{tikzpicture}
      \node (quotableFormu) at (0, 0)  {$\quotableFormu$};
   \node (L) at (0, 1) {$\mathcal L$};
   \draw[includedin] (quotableFormu) -- (L);
\end{tikzpicture}
    }
    \caption{Inclusion relationship ($\hookrightarrow$) between subsets of the set of terms $\termset$ and subset $\quotableFormu$ of $\mathcal L$.}
    \label{fig:venndiagramterms}
\end{figure}

\subsection{Quoting}


We now define the quotation function $\mu$. This function takes a quotable formula or term and outputs its quotation.
We define~$\mu$ jointly with the sets of terms and formulas it can be applied to.

The set of quotable terms $\quotableTerm$ is the set of terms present in quotable formulas. These terms contain only variables from $V$ and are recursively formed with non-quotation symbols (the non-underlined symbols) or with the image by $\mu$ of quotable elements.
By quotable elements, we mean elements of $\quotableTerm$ or $\quotableFormu$.
The set of quotable formulas $\quotableFormu$ is the set of formulas for which we can produce quotations. 
They contain only variables in $V$, do not contain the predicate $\truth$, and use only terms from $\quotableTerm$.
These are the formulas we can handle as objects of the logic, the formulas we can say are true or false in a context.

We define $\quotableTerm$, $\quotableFormu$, and $\mu$ jointly by mutual induction.
\begin{definition} The sets $\quotableTerm, \quotableFormu$ are defined inductively by
    \begin{align*}
        & \quotableTerm \hspace*{0.1cm} := x \mid t_q \mid f(t_1, \dots, t_n)\\
        & ~~~~~~~~~\text{\emph{for}}~~ x \in V, f \in F_b, t_1, \dots, t_n \in \quotableTerm, t_q \in \mu(\quotableTerm \cup \quotableFormu)\\
        & \quotableFormu \hspace*{0.1cm} := p(t_1, \dots, t_n) \mid \forall x.\ \varphi \mid \neg \varphi \mid \varphi_1 \land \varphi_2 \\
        & ~~~~~~~~~\text{\emph{for}}~~ p \in P_b, t_1, \dots, t_n \in \quotableTerm, x \in V, \varphi, \varphi_1, \varphi_2 \in \quotableFormu\\
    \end{align*}
\end{definition}
\begin{definition}
    $\mu$ is inductively defined on $\quotableTerm \sqcup \quotableFormu$ by:
    \begin{align*}
        \mu(f(t_1, \dots, t_n)) & := \fquo f (\mu(t_1),\dots,\mu(t_n)) \\ 
        \mu(t_q) & := \quot(t_q) \\
        \mu(p(t_1, \dots, t_n)) & := \fquo p (\mu(t_1), \dots, \mu(t_n)) \\
        \mu(x) & := \fquo x\\
        \mu(\phi_1 \land \phi_2) & := \bland(\mu(\phi_1),\mu(\phi_2)) \\
        \mu(\neg \phi) & := \bneg(\mu(\phi)) \\
        \mu(\forall x.\ \phi) & := \bforall (\fquo x, \mu(\phi))
    \end{align*}  
    \textbf{Range:} $f \in \fq$, $x \in V$, $t_1,\dots,t_n \in \quotableTerm$, $t_q \in \tq$.
\end{definition}

For ease of reading, whenever $\mu(\varphi)$ is defined, we write it as $\bm \varphi$, underlining the entire argument. For instance, we write $\fquo{p(x)}$ instead of $\fquo{p}~(\fquo{x})$. Instead of $\fquo{ist}(\fquo{x}, \quot(\fquo{happy}(\fquo{y})))$, we write $\bmnested{ist(x, \bmnested{happy(y)})}$.

\begin{example}
    The term $1+x$ with $x \in V$ is quotable, i.e., $1+x~~\in~~\quotableTerm$. 
    $1+y$ with $y \in V_{\infty} \setminus V$ is not. 
    The term $1 \fquo{+} x$ is not quotable because of the symbol $\fquo{+}$; but the term $\fquo{1}\fquo +\fquo x$ is quotable, because it is the image by $\mu$ of $1+x$, which is quotable.
\end{example}
\begin{example}
    The formula $p(x)$ with $x \in V$ is quotable, i.e., $p(x) \in \quotableFormu$. 
    The formula $P(1\bm +x)$ is not quotable since $1\bm +x$ is not a quotable term.
    The formula $p(\fquo 1)$ is quotable. 
\end{example}
\noindent Note that whenever we quote a formula that already contains a quotation, we put said quotation in the $\quot$ symbol, to add a level of quotation:
 
\begin{example}
    $\bmnested{P(\bmnested{1=1})} = \mu(P(\bmnested{1=1})) = \bmnested P(\quot(\bmnested{1=1}))$. 
\end{example}

\noindent The formulas in $\quotableFormu$ do not use the predicate $\truth$, it can never be quoted. This is an important restriction that avoids the complications of Tarski's Theorem of the undefinability of Truth (see Proposition~\ref{th:addTruth} in Section~\ref{sec:def_qiana}).

\subsection{Substitution on Quotations}

\noindent In our axiomatization, we also need a substitution function that substitutes not variables but quoted variables. Its definition is similar to the standard substitution on first-order logic formulas:
\begin{definition}
\label{def:substitutionquotations}
Given $t$ in $\termset$, given $x$ in $V$, we define $z[\fquo{x} \leftarrow t]_q$ on $z \in \tq$ inductively as:
    \begin{align*}
        \bm f(t_1, \dots, t_n)[\fquo{x} \leftarrow t]_q & = \bm f(t_1[\fquo{x} \leftarrow t]_q, \dots, t_n[\fquo{x} \leftarrow t]_q) \\
         \fquo{x}[\fquo{x} \leftarrow t]_q & = t \\
        \bm p(t_1, \dots, t_n)[\fquo{x} \leftarrow t]_q & = \bm p(t_1[\fquo{x} \leftarrow t]_q, \dots, t_n[\fquo{x} \leftarrow t]_q) \\
        \fquo{y}[\fquo{x} \leftarrow t]_q & = \fquo{y} \\
        \bforall(t_1, t_2)[\fquo{x} \leftarrow t]_q & = \bforall(t_1[\fquo{x} \leftarrow t]_q,t_2[\fquo{x} \leftarrow t]_q) \textnormal{ if } t_1 \neq \fquo{x} \\
        \bforall(\fquo{x}, t_2)[\fquo{x} \leftarrow t]_q & = \bforall(\fquo{x}, t_2) \\
        t_1[\fquo{x} \leftarrow t]_q & = t_1 \textnormal{ in all other cases}
    \end{align*}
    \textbf{Range:} $\bm x, \bm y \in \cv$ ($\bm x \neq \bm y$), $t_1, \dots, t_n \in \quotableTerm, \bm f \in \fq, \bm p \in \fpq$
\end{definition}
\noindent This quoted substitution function works just like the standard substitution:
\begin{example} Applying the quoted substitution function:
   $$\fquo p (\fquo x) \bland \bforall (\fquo y, \fquo q(\fquo x, \fquo y))[\fquo x \leftarrow 1+1]_q = \fquo p (1+1) \bland \bforall (\fquo y, \fquo q(1+1, \fquo y))$$
    
\end{example}



\subsection{Unquoting}

\noindent We are now ready to define the converse of the quoting operator $\mu$:
\begin{definition} \label{def:unquoteRecu} \label{def:unquoting}
The unquote operator $\uq$ is defined on $\tqv$ by the following recursive definition.
\begin{align*}
     \uq(\bm f(t_1, \dots, t_n)) & = f(\uq(t_1), \dots, \uq(t_n)) \\
     \uq(\bm p(t_1, \dots, t_n)) & = p(\uq(t_1), ..., \uq(t_n)) \\
     \uq(quote(t)) & = t \\
     \uq(\bm \varphi_1 \bland \bm \varphi_2) & = \uq(\bm \varphi_1) \land \uq(\bm \varphi_2) \\
     \uq(\bneg \bm \varphi) & = \neg \uq(\varphi) \\
     \uq(\bforall(\bm x,\bm \varphi)) & = \forall x.\ \uq(\bm \varphi[\bm x \leftarrow quote(x)]_q) \\
     \uq(\bm x) & = x \\
     \uq(t) & = t \textnormal{ in all other cases}
\end{align*}
\textbf{Range:} $\bm f \in \fq$, $\bm p \in \fpq$, $x \in V$, $t, t_1, \dots, t_n \in \tqv$
\end{definition}

The reason we use the notation $\uq$ for our unquote operator is that it is the inverse of $\mu$ on the image of $\mu$.
$\uq$ is not limited to the inverse of $\mu$, but it extends it.

\begin{proposition}
    $\mu$ is injective on $\mathcal{L}_q$.
\end{proposition}
\begin{proof}
    We can prove by induction on $(\alpha, \beta) \in (\quotableTerm \sqcup \quotableFormu)^2$ that $\mu(\alpha) = \mu(\beta)$ implies $\alpha = \beta$.
\end{proof}

\begin{proposition} \label{prop:unquoteRecu}
    $\forall x \in \quotableFormu \cup \quotableTerm, \uq(\mu(x)) = x$
\end{proposition}
\begin{proof}
    This is proven by induction on $x \in \quotableFormu \cup \quotableTerm$.
\end{proof}   

\noindent Intuitively, $\mu$ adds a level of underlining on quotable formulas and terms, and $\uq$ removes it:

 \begin{example}
 $\uq(\bmnested{ist(x, \bmnested{happy(y)})}) = ist(x, \bmnested{happy(y)})$
 \end{example}

\section{Defining Qiana} \label{sec:axioms} 
    \label{sec:def_qiana}


We can now define Qiana and its core axioms.
The predicate $\truth$ is designed to say that its argument is true in reality, as given in the following axiom schema (for all $\varphi \in \quotableFormu$):
\begin{align} 
    & \truth(\mu(\varphi)) \leftrightarrow \varphi \label{def:oldDefT} \tag{A$_\text{truth}$}
\end{align}
\noindent As famously shown by Alfred Tarski~\cite{Tarski1936-TARTCO}, this form of predicate can lead to self-referential formulas and incoherent theories (see \cite{BookLogic} 
for a more modern description). Here, the fact that we did not allow the quotation of \truth will protect us from the pitfalls of Tarski's theorem. 
We show this with Proposition~\ref{th:addTruth}:


\begin{proposition}\label{th:addTruth}
    Let $S^{-\truth}$ be the signature equal to $S$ without the symbol \truth.
    Let $H^{-\truth}$ be a coherent theory under the signature $S^{-\truth}$. Let $H$ be the closure of $H^{-\truth}$ under schema~\ref{def:oldDefT}. $H$ is coherent.
\end{proposition}
\begin{proof}
    Let $M^{-\truth}$ be a model of $H^{-\truth}$. We define $M$ (the model of $H$) as equivalent to $M^{-\truth}$ on all symbols except~$\mathbf{T}$. For each $\bm \varphi$ in $\tqf$ we check whether $\uq(\bm\varphi)$ is true under $M^{-\truth}$; $M \models \truth(\bm \varphi)$ if and only if that is the case. 
\end{proof}

\newcommand{\AoneAfour}{\ensuremath{A1\textnormal{-}4}\xspace}

\subsection{Truth Axioms}

The axiom schema \ref{def:oldDefT} is not explicitly in a Qiana theory. 
It will be subsumed by axiom schemas~\ref{ax:str1} to~\ref{ax:str4}, which conveniently admit direct counterparts in the finite axiomatization process of Section~\ref{subsec:infiToFinFiniteAxio}. Here,
$x_1, ..., x_n$ are distinct variables.
\begin{align}
    & \forall x_1, \dots, x_n.\ \truth(\bm p(t_1,\dots, t_m)) \leftrightarrow  p(\uq(t_1), \dots, \uq(t_m)) \label{ax:str1} \tagqianaaxiom \\
    & \forall x_1, \dots, x_n.\ \truth(A \bland B) \leftrightarrow (\truth(A) \land \truth(B)) \label{ax:str2} \tagqianaaxiom \\
    & \forall x_1, \dots, x_n.\ \truth(\bneg A) \leftrightarrow (\neg \truth(A)) \label{ax:str3} \tagqianaaxiom \\
    & \forall x_1, \dots, x_n.\ \truth(\bforall(\fquo{x} ,A)) \leftrightarrow  (\forall x.\ \truth(A[\fquo{x} \leftarrow \quot(x)]_q)) \label{ax:str4} \tagqianaaxiom
\end{align}
\textbf{Range:} $p \in P, t_1, \dots, t_n \in \tqtv, A, B \in \tqv, x \in V$\\
\noindent Schema~\ref{ax:str1} concerns the truth of the quotation of an atomic formula. 
Schema~\ref{ax:str2} and Schema~\ref{ax:str3} are about the Boolean connectives. 
Schema \ref{ax:str4} handles universal quantification through substitution and the $\quot$ predicate; we illustrate this behavior in Example~\ref{ex:TschemaApplication} below.
Note that, different from~\ref{def:oldDefT}, Schemas \AoneAfour apply also to non-well-formed terms.

\begin{example}
The formulas $\truth(\bneg \quo{P}()) \leftrightarrow \neg \truth(\quo{P}())$ and $\truth(\bneg 2) \leftrightarrow \neg \truth(2)$ are both instances of \ref{ax:str3}.
\end{example}

\noindent Axiom schema \ref{ax:str4} is the treatment of the universal quantification. We substitute the quotation $\fquo x$ of a variable $x$ by $\quot(x)$. Indeed, $\fquo x$ is a constant symbol, and this mechanism enables to effectively simulate the quantification through a quotation. 
The following example illustrates this mechanism.

\begin{example} \label{ex:TschemaApplication}
We show $\AoneAfour \hspace*{-0.1cm} \models \hspace*{-0.1cm} \truth(\bforall(\fquo{x} ,\fquo{P}(\fquo{x}))) \hspace*{-0.1cm}\leftrightarrow \hspace*{-0.1cm} \forall x.\ P(x)$. \\
To prove this, let $M$ be a model of $\AoneAfour$. We have:
    \begin{align*}
        & M \models \truth(\bforall(\fquo{x} ,\fquo{P}(\fquo{x}))) \\
        & \textiff M \models \forall x.\ \truth(\fquo{P}(\bm x)[\bm x \leftarrow \quot(x)]_q) & \text{as $M \models \ref{ax:str4}$} \\
        & \textiff M \models \forall x.\ \truth(\fquo{P}(\quot(x))) \\
        & \textiff M \models \forall x.\ P(\uq(\quot(x))) & \text{ as $M \models \ref{ax:str1}$} \\
        & \textiff M \models \forall x.\ P(x) & \text{ by Definition~\ref{def:unquoting}}
    \end{align*}
\end{example}
\noindent We are now ready to prove that any instance of  \ref{def:oldDefT} is a logical consequence of the theory \ref{ax:str1}-\ref{ax:str4}.

\begin{proposition} \label{prop:truthSubsumption}
$\AoneAfour \models$ \ref{def:oldDefT}. 
\end{proposition}
\begin{proof}
    We prove this via induction on the following property: 
    Let $A \in \tqfv$ with no free quoted variables (i.e.,~each $\bm x$ is quantified by a $\bforall$).
    Let $x_1, \dots, x_n$ be the free variables of $A$.
    Then \[\AoneAfour \models \forall x_1, \dots, x_n.\ \truth(A) \leftrightarrow \uq(A)\]
    We detail the proof in the supplementary material.
\end{proof}

\subsection{Axioms for Reasoning in Contexts} \label{subsec:reasoningAxioms}

Reasoning axioms endow contexts with some inference power. They say that contexts that ``know'' some things must also know some direct consequences of said things.
We first introduce a few general schemas to this end, including associativity, commutativity, and distributivity.
For example, schema~\ref{ax:qiana1} tells us that in any context (represented by variable $x_c$) if the conjunction of two formulas is true (represented by their quotations through variables $x_1$ and $x_2$), then the first of these formulas is also true.
%
\begin{align}
    & \forall x_c, x_1, x_2.\ \textit{ist}(x_c,x_1 \bland x_2) \rightarrow \textit{ist}(x_c, x_1) \label{ax:qiana1} \tagqianaaxiom\\
    & \forall x_c,x_1,x_2.\ \textit{ist}(x_c, {x_1 \bland x_2}) \leftrightarrow \textit{ist}(x_c,  {x_2 \bland x_1}) \label{ax:qiana2}  \tagqianaaxiom \\
    & \forall x_c,x_1.\ \textit{ist}(x_c, {\bneg \bneg x_1}) \leftrightarrow \textit{ist}(x_c, {x_1}) \label{ax:qiana3}  \tagqianaaxiom \\
    & \forall x_c,x_1,x_2,x_3.\ \textit{ist}(x_c, {(x_1 \bland x_2) \bland x_3})\leftrightarrow \textit{ist}(x_c, {x_1 \bland (x_2 \bland x_3)}) \label{ax:qiana4}  \tagqianaaxiom\\
    & \forall x_c,x_1,x_2,x_3.\ \textit{ist}(x_c, {(x_1 \bland x_2) \blor x_3}) \leftrightarrow \textit{ist}(x_c, {(x_1 \blor x_3) \bland (x_2 \blor x_3)}) \label{ax:qiana5}  \tagqianaaxiom
\end{align}
Other properties of associativity, commutativity, and distributivity can be deduced from the above, also with the help of the definition of $(a \blor b)$ as $\bneg (\bneg a \bland \hspace*{0.05cm} \bneg b)$. Next, we introduce the disjunctive syllogism (\textit{modus ponens}), which says that if an agent knows $\phi$ and $\phi \Rightarrow \psi$, then it also knows $\psi$:
\begin{align}
   &\forall x_c, x_1, x_2.\ \textit{ist}(x_c, x_1 \blor x_2) \land \textit{ist}(x_c, {\bm \neg} x_1) \rightarrow \textit{ist}(x_c,x_2) \label{ax:istMP} \tagqianaaxiom
\end{align}
\noindent We also introduce an axiom schema that gives a context some ability to handle $\forall$: A quoted formula can be replaced by each of its instantiations.
\begin{align} 
    & \forall c.\ \textit{ist}(c, \bforall(\fquo{x},\bm{\varphi})) \rightarrow \forall x.\ \textit{ist}(c, \bm{\varphi}[\fquo{x} \leftarrow \quot(x)]_q) \label{as:oldDefss} \tagqianaaxiom
\end{align}
\textbf{Range:} $x \in V$, $\bforall(\fquo{x}, \bm \varphi) \in \tqf$ \\

\noindent We illustrate the use of schema~\ref{as:oldDefss} with the example of Romeo and Juliet from the introduction (\textit{Cap} stands for being a member of Juliet's family, the Capulets):
\begin{example}~\label{ex:allCapuletsAreNice} Juliet believes that all Capulets are nice.
    \[\textit{ist}(\textit{believes}(J), \fquo{\forall x.\ \textit{Cap}(x) \rightarrow \textit{nice}(x)}) \rightarrow \forall x.\ \textit{ist}(\textit{believes}(J), \fquo{\textit{Cap}}(\quot (x)) \fquo{\rightarrow}\ \fquo{\textit{nice}}(\quot(x)))\]
\end{example}

\subsection{Qiana}

We can now formally define a Qiana-closure theory: 

\begin{definition}[Qiana-closure theory] \label{def:qiana}
Let $H$ be a theory. The Qiana-closure of $H$, denoted by $\hq$, is the theory
$$\hq = H  \cup \textrm{A1-A11}.$$
\end{definition}
\noindent As an immediate property, we have semi-decidability: 
\begin{proposition}[Semi-decidability]
If a theory $H$ is recursively enumerable, 
then the problem of deciding whether its Qiana closure $\hq$ entails some formula $\varphi$ is semi-decidable.
\end{proposition}

\begin{proof}
    Axiom schemas $\textrm{A1-A11}$ are recursive. Any recursively enumerable theory $\Sigma$ leads to semi-decidability of the entailment problem: given $\phi$, decide whether $\Sigma \models \phi$.
\end{proof}


\section{Examples and Discussion} \label{sec:Example}
    Let us now present some examples to illustrate how Qiana can be used to model epistemic knowledge. 

\subsection{Reasoning in Epistemic Contexts} \label{subsec:suicideExample}

Let us now reconsider the example of Romeo and Juliet from the introduction.
For simplicity's sake, we will not model time in this example. This choice will result in seemingly absurd simultaneity but should not hamper understanding. 
We start with our hypotheses from the introduction. We skip the description of suicide and consider death a direct consequence of believing one's love to be dead. 
\begin{align}
        & \forall \phi.\ \textit{ist}(\textit{says}(\textit{FriarLaurence}),\phi) \rightarrow \truth(\phi)\label{eq00}\\
        & \forall x,y.\ \textit{madlyLoves}(x,y) \land \textit{ist}(\textit{believes}(x), \quo{dead}(y)) \rightarrow \textit{dead}(x) \label{eq0}
\end{align}
Next, we state some obvious facts from the tragedy:
\begin{align}
    & \textit{madlyLoves}(\textit{Romeo, Juliet})\label{eq1} \\
    & \textit{madlyLoves}(\textit{Juliet, Romeo})\label{eq2} \\
    & \textit{ist}(\textit{says}(\textit{FriarLaurence}), \bforall(\bm{x},\quo{drinkPotion}(\bm{x}) \ \bm{\rightarrow}\ \quo{appearDead}(\bm{x}))) \label{eq3} \\
    & \textit{drinkPotion}(\textit{Juliet}) \label{eq4}
\end{align}
Finally, we need some world knowledge: by definition, people can see if someone appears dead. They can also see if someone is dead.
\begin{align}
    & \forall c, x.\ \textit{appearDead}(x) \ \rightarrow{} \ \textit{ist}(c, \quo{appearDead}(x)) \label{eq5}\\
    & \forall x,y.\ \textit{dead}(y) \rightarrow \textit{ist}(x,\quo{dead}(y))\label{eq5b}
\end{align}
The next hypothesis is perhaps best summed up as ``Romeo does not know how to check someone's pulse'':
\begin{equation} 
    \forall x. \textit{ist}(\textit{believes}(\textit{Romeo}), \quo{appearDead}(x) \rightarrow \quo{dead}(x)) \label{eq6}
\end{equation}
\noindent We can now see the tragedy unfold:
\begin{align}
   &\forall x.\textit{drinkPotion}(x) \rightarrow \textit{appearDead}(x)~\text{from~\ref{eq00},~\ref{eq3}}\label{eq7}\\
   &\textit{appearDead}(\textit{Juliet})  ~~\text{from~\ref{eq4} and~\ref{eq7}}\label{eq8}\\
   &\textit{ist}(\textit{believes}(\textit{Romeo}), \quo{appearDead}(\quo{Juliet})) ~~\text{from~\ref{eq5},~\ref{eq8}}\label{eq9}\\
   &\textit{ist}(\textit{believes}(\textit{Romeo}), \quo{dead}(\quo{Juliet}))  ~~\text{from~\ref{eq9} and~\ref{eq6}}\label{eq10}\\
   &dead(\textit{Romeo})  ~~\text{from~\ref{eq10},~\ref{eq1}, and~\ref{eq0}}\label{eq11}\\
   &\textit{ist}(\textit{believes}(\textit{Juliet}), \quo{dead}(\quo{Romeo}))  ~~\text{from~\ref{eq11} and~\ref{eq5b}}\label{eq12}\\
   &\textit{dead}(\textit{Juliet})  ~~\text{from~\ref{eq12},~\ref{eq0}, and~\ref{eq2}}
\end{align}

\subsection{Paraconsistency}\label{sec:para}
In first-order logic, an inconsistent theory can be used to deduce anything: if $H \vdash (\varphi \land \neg \varphi)$ then $H \vdash \textit{alive}(\textit{Elvis})$.
This phenomenon is called \emph{the principle of explosion}. While this is still true in Qiana theories, it is not true of the beliefs modeled inside contexts: A context can contain both a statement and its negation, and no axiom schema of Qiana allows deducing arbitrary statements from such beliefs (neither inside the context nor outside).
This can be useful, e.g., to model contradictory beliefs. 
In our running example of Romeo and Juliet, let us assume for a moment that Romeo did notice that Juliet had a pulse but did not conclude that Juliet was alive.
%
\begin{align*}
     H :=\ & \{ \textit{dead}(\textit{Juliet}), \textit{hasPulse}(\textit{Juliet}), \\
    & \hspace{-0.75cm} \forall x.\ \neg (\textit{alive}(x) \land \textit{dead}(x)), \forall x.\ \textit{hasPulse}(x) \rightarrow \textit{alive}(x) \}
\end{align*}
In normal first-order logic, this is an inconsistent theory, and it can thus be used to deduce anything: $H \vdash \textit{alive}(\textit{Elvis})$.
%
%
Qiana, in contrast, emulates a paraconsistent logic inside contexts. Hence, the principle of explosion does not apply inside contexts: 
%
\begin{align*}
     H' :=\ & \{ \textit{ist}(\textit{believes}(\textit{Romeo}), \quo{dead}(\quo{Juliet})) \\
    & \textit{ist}(\textit{believes}(\textit{Romeo}), \quo{hasPulse}(\quo{Juliet}))
\end{align*}
\begin{align*}
    & \textit{ist}(\textit{believes}(\textit{Romeo}), \bforall(\quo{x}, \bneg(\quo{alive}(\quo{x}) \bland \quo{dead}(\quo{x})))) \\
    & \textit{ist}(\textit{believes}(\textit{Romeo}), \bforall(\quo{x}, \quo{hasPulse}(\quo{x})\ \underline{\rightarrow}\ \quo{alive}(\quo{x}))) \}
\end{align*}
\noindent These contradictory thoughts now entail:
\begin{align*}
    H' & \vdash \textit{ist}(\textit{believes(Romeo)},\quo{dead}(\quo{Juliet}))\\
    H' & \vdash \textit{ist}(\textit{believes(Romeo)}, \bneg \quo{dead}(\quo{Juliet}))
\end{align*}
\noindent However, they do not imply that Romeo believes anything:
\begin{align*}
    H' \not\vdash \textit{ist}(\textit{believes(\textit{Romeo})},\quo{alive}(\quo{Elvis}))
\end{align*}
%
\noindent If we want to keep the principle of explosion inside contexts, we can 
add the following axiom schema to our theories (for all $\varphi, \psi \in \quotableFormu$):
\begin{align} 
& \forall c.\ ist(c,\bm \varphi) \rightarrow ist(c, \bm \varphi \blor \bm \psi)
\end{align} 
Together with \textit{modus ponens}, it allows to deduce anything from a contradiction.
From $ist(c, \varphi)$ we deduce $ist(c, \varphi \blor \psi)$. From $ist(c, \bneg \varphi)$ and $ist(c, \varphi \blor \psi)$ we deduce $ist(c, \psi)$.

\subsection{Mixing Different Types of Contexts}\label{subsec:types_contexts}
Until now, we have shown how to use contexts to model beliefs, and we have considered the play itself as the truth. 
However, contexts can also encapsulate a story.
To illustrate this, let us consider two versions of the story of Romeo and Juliet: the original and a fanfiction variant. 
In the fanfiction variant, Romeo decides to check Juliet's pulse and notices that she is alive. He waits for her to come to her senses and then they leave together and live happily ever after. 

We start by declaring that the fanfiction and the original are both stories:
\begin{align*}
    & \textit{story}(\textit{fanfiction}) \land \textit{story}(\textit{original})
\end{align*}
In the fanfiction, Romeo checks Juliet's pulse; in the original, he does not:
\begin{align*}
    & ist(\textit{fanfiction}, \quo{\textit{checkPulse}}(\quo{R}, \quo{J})) \\
    & ist(\textit{original}, \bneg \quo{\textit{checkPulse}}(\quo{R},\quo{J}))
\end{align*}
In all stories, if Romeo checks Juliet's pulse, he knows she is alive.
In all stories, for all persons, if Romeo does not feel their pulse and they appear dead, he thinks they are dead.
\begin{align*}
    & \forall s.\ \textit{story}(s) \rightarrow ist(s, \bmnested{\textit{checkPulse}(R, J) \rightarrow ist(\textit{believes}(R), \bmnested{alive(J)})}) \\
    & \forall s.\ \textit{story}(s) \rightarrow ist(s, \bmnested{\forall x.\ \textit{appearDead}(x) \land \neg \textit{checkPulse(R,x)} \rightarrow ist(\textit{believes}(R), \bmnested{\neg alive(x)})}) \\
\end{align*}
\noindent In all stories, Juliet appears dead.
In all stories, if Romeo knows Juliet is alive, he does not kill himself, but he does if he thinks she is dead. In all stories, the protagonists will be either both dead or both alive: 
\begin{align*}
    & \forall s.\ \textit{story}(s) \rightarrow ist(s, \fquo{appearDead(J)}) \\
    & \forall s.\ \textit{story}(s) \rightarrow ist(s, \bmnested{ist(\textit{believes}(R), \fquo{alive(J)}) \rightarrow alive(R)}) \\
    & \forall s.\ \textit{story}(s) \rightarrow ist(s, \bmnested{ist(\textit{believes}(R), \bmnested{\neg alive(J)}) \rightarrow \neg alive(R)})
\end{align*}
\noindent These formulas -- together with the axioms of Qiana -- are enough to deduce that there is the usual ending in the original version of the story and a vastly more fortunate one in the fanfiction story:
\begin{align*}
& \ist(\textit{original}, \fquo{\neg \textit{alive}(R)})\\
    & \ist(\textit{fanfiction}, \fquo{\textit{alive}(R)})
\end{align*}

\subsection{On the Use of the Quote Symbol $\quot$}

The quote symbol\footnote{$\textit{quote}$ in our previous article~\cite{QianaKR24}} $\quot$ is used to inject a value directly into a quotation.
When we apply the unquote operator $\uq$ to a quotation containing $\quot$, the content of $\quot$ will not be unquoted. For example:
\[\uq(\bm2) = \uq(\quot(2)) = 2\]
\[\uq(\bm P(\quot(2))) = P(2)\]


To avoid any confusion, we recall that $\quot$ is a symbol of the logic while $\mu$ is a meta operator called the `quotation operator''.
$\mu$ exists outside the logic itself and returns the quotation of a given formula or term, preserving its structure.
Most of the time, we write $\bm \varphi$ instead of $\mu(\varphi)$. 
By contrast, $\quot$ is just a function symbol within the logic. It represents a function of the domain of discourse, just like any other function symbol in FOL.

\noindent The two main purposes of $\quot$ are:
\begin{enumerate}
    \item Nesting quotations
    \item Using a variable inside a quotation
\end{enumerate}

\noindent We first illustrate nested quotations.
Let us say that Romeo believes that Juliet believes that he is smart.
We can write that Juliet believes that Romeo is smart as follows: 
\[\ist(\textit{believes}(J), \bm{\textit{Smart}}(\bm{R}))\]

To say that Romeo believes that Juliet believes that he is smart, we need to quote the formula above. To this end, we need the symbol $\quot$:
\[\ist(\textit{believes}(R), \bm{\ist}(\bm{\textit{believes}}(\bm{J}), \quot(\bm{\textit{Smart}}(\bm{R}))))\]

\noindent Using our notations, we could write this more compactly as follows:
\[\ist(\textit{believes}(R), \bmnested{\ist(\textit{believes}(J), \bmnested{\textit{Smart}(R)})})\]

\noindent If we instead wrote the formula without $\quot$, we would end up with:
\[\ist(\textit{believes}(R), \bm{\ist}(\bm{\textit{believes}}(\bm{J}), \bm{\textit{Smart}}(\bm{R})))\]

But $\bm{\ist}(\bm{\textit{believes}}(\bm{J}), \bm{\textit{Smart}}(\bm{R}))$ is not a well-formed quotation. If we were to try to unquote it, we end up with the formula below.
It is not a well-formed first-order formula, as $\textit{Smart}$ is a predicate and, therefore, cannot appear in the argument of the predicate $\ist$.
\[\uq(\bm{\ist}(\bm{\textit{believes}}(\bm{J}), \bm{\textit{Smart}}(\bm{R}))) = \ist(\textit{believes}(J), \textit{Smart}(R))\]

\noindent Hence we need the symbol $\quot$ to nest quotations.

Now we illustrate the second use of $\quot$: using a variable inside a quotation.
Let us say that Romeo is very naive: all liars have successfully convinced him that they are honest. This does not mean that Romeo believes that ``all liars are honest'', which would be written as:
\[\ist(\textit{believes}(R), \bforall(c_x, \bm{\textit{Liar}}(c_x) \bm\rightarrow\ \bm{\textit{Honest}}(c_x)))\]

\noindent Instead, we want to say that for all liars, Romeo believes that person is honest.  We can write this as:
\[\forall x.\ \textit{Liar}(x) \rightarrow \ist(\textit{believes}(R), \bm{\textit{Honest}}(\quot(x)))\]

\noindent Without the symbol $\quot$, we might try to write the following:
\[\forall x.\ \textit{Liar}(x) \rightarrow \ist(\textit{believes}(R), \bm{\textit{Honest}}(x))\]
But this does not work. We can notice this by instantiating the quantifier $\forall$ in the formula above.
\[\textit{Liar}(\textit{Juliet}) \rightarrow \ist(\textit{believes}(R), \bm{\textit{Honest}}(\textit{Juliet}))\]
This makes no sense because $\bm{\textit{Honest}}(\textit{Juliet})$ is not a well-formed quotation.

\section{Finite Axiomatization and Theorem Provers} \label{sec:FiniteAxio}
    



\newcommand{\iconFiniteAxioms}{\faIcon[regular]{file}}
\newcommand{\iconInfiniteAxioms}{\raisebox{-0.2mm}{\includegraphics[height=3mm]{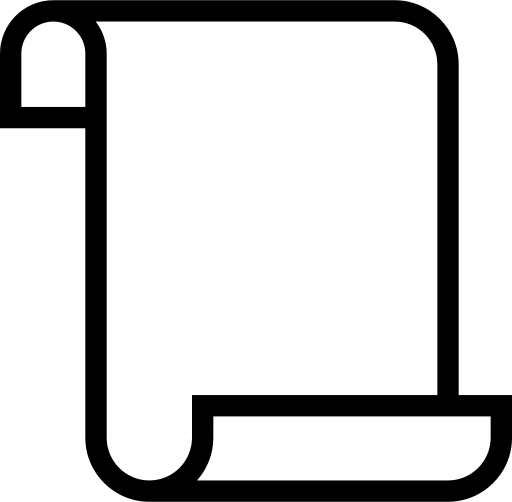}}}
\newcommand{\figureInfiniteAxioms}[2]{\!\!\tikz[baseline=1mm]{\node {\begin{tabular}{c} 
#2 \\
{\scriptsize \ensuremath{#1}}
\end{tabular}};}\!\!
}

\newcommand{\figureFiniteAxioms}[2]{\figureInfiniteAxioms{#1}{#2}}

    
    

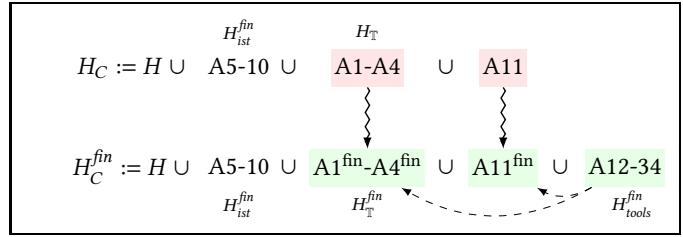
\begin{figure}
\newcommand\tikzNode[2]{\tikz[remember picture, baseline=0mm]{ \node[inner sep=0mm] (#1) {#2};}}
   \fbox{ 
    \centering
\newcommand{\yhcfin}{-1.4}
\tikzstyle{finite} = [inner sep=0.5mm, fill=green!10!white]
\tikzstyle{infinite} = [inner sep=0.5mm, fill=red!10!white]
\tikzstyle{transformation} = [-latex,decorate,decoration={zigzag,segment length = 2mm, amplitude = 0.3mm}, line width=0.2mm]
\tikzstyle{help} = [dashed, -latex]

\scalebox{0.95}{
    \begin{tikzpicture}
        \node  (HC) at (0, 0) {\strut $\hq := H~\cup$};
        \node (HCfin) at (0, \yhcfin+0.02) {\strut $\hf := H~\cup$};
        \node (A510) at (1.5, 0) {\strut A5-10};
        \node (A510fin) at (1.5, \yhcfin) {\strut A5-10};
        \node at (2.2, 0) {\strut $\cup$};
        \node  at (2.2,\yhcfin) {\strut $\cup$};
         \node[infinite] (A1A4) at (3.3, 0) {\strut A1-A4};
        \node[finite] (A1A4fin) at (3.3, \yhcfin) {\strut A1\expofinite-A4\expofinite};
          \node at (4.4, 0) {\strut $\cup$};
        \node  at (4.4, \yhcfin) {\strut $\cup$};
            \node[infinite] (A11) at (5.2, 0) {\strut A11};
        \node[finite] (A11fin) at (5.2, \yhcfin) {\strut A11\expofinite};
        \node  at (6, \yhcfin) {\strut $\cup$};
        \node[finite] (A1234) at (6.9, \yhcfin) {\strut A12-34};
        \node at (1.5, 0.5) {\scriptsize $\hqf$};
         \node at (1.5, \yhcfin-0.5) {\scriptsize $\hqf$};
        \node at (3.3, 0.5) {\scriptsize $\htr$};
        \node at (3.3, \yhcfin-0.5) {\scriptsize $\htf$};
         \node at (7, \yhcfin-0.5) {\scriptsize $\hof$};

         \draw[help] (A1234) edge[bend left] 
         (A1A4fin);
         \draw[help] (A1234) edge[bend left]
         (A11fin);
\draw[transformation] (A11) --
(A11fin);
\draw[transformation] (A1A4) -- 
(A1A4fin);
         
    \end{tikzpicture}}
    }
    
    \caption{Overview of the process of finite axiomatization. Zizag arrows indicate the finite sets are counterparts to the top infinite ones. Dotted arrows indicate the the elements of $\hof$ are used to define said counterparts.}
    \label{fig:finiteaxiomatization}
\end{figure}
 
We will now show how the Qiana closure of any finite theory can be finitely axiomatized. 
Let $H$ be a given finite theory on a quotation-compatible signature $S$. 
The Qiana-closure of $H$ is $\hq = H \cup \text{A1-A4} \cup \text{A5-A10} \cup \text{A11}$ (see Definition~\ref{def:qiana}). Both $\text{A1-A4}$ and $\text{A11}$ are infinite. Hence, $\hq$ is also infinite.
In this section, we will present a process to define another theory $\hf$ that is both finite and equisatisfiable with $\hq$.
This will allow us to test the satisfiability of $\hq$ by feeding finitely many formulas (the elements of $\hf$) to a theorem prover.
Subsections~\ref{subsec:extendFiniteAxio} and~\ref{subsec:infiToFinFiniteAxio} will introduce secondary sets of use in this finite axiomatization. Subsection~\ref{subsec:propsFiniteAxio} concludes the presentation of this process and gives the relevant results.

We write $\hqf = \text{A5-A10}$.
Fortunately, $\hqf$ is finite. However, 
we must replace the infinite axiom schemas $\text{A1-A4}$ and $\text{A11}$. 
To do so, we extend the signature $S$ to another signature $S'$, which contains new symbols that we describe with additional (finite) schemas.
Together, these new schemas form the set $\hof$. We present these symbols and the set $\hof$ in Subsection~\ref{subsec:extendFiniteAxio}.
Based on these symbols and their definition schemas, we introduce finite counterparts to our infinite schemas in Subsection~\ref{subsec:infiToFinFiniteAxio}.
These are the sets $\htf$ and $\hfoinf$.
Together, these sets allow us to define a new set $\hf = H \cup \text{A5-A10} \cup \hof \cup \htf \cup \hfoinf$, which is finite and equisatisfiable with $\hq$ (see Figure~\ref{fig:finiteaxiomatization}). 
We present some interesting properties of the process in Subsection~\ref{subsec:propsFiniteAxio}, the most important being the correctness of the process.

\subsection{Utility Symbols for the Finite Axiomatization} \label{subsec:extendFiniteAxio}

We will now introduce new symbols that will allow us to define the finite axiomatization of Qiana.
We consider a fixed and finite theory $H$ on $S$.
Without loss of generality, we introduce two fresh function symbols $\sub$ and $\E$ and three predicate symbols $\wft$, $\reach$, and $=$. 
By adding these symbols to $S$ we obtain a larger signature $S'$. The sets $F_b$ and $P_b$ are unchanged by these additions, but for the rest of this section, the sets $F$ and $P$ are the sets of function and predicate symbols in $S'$.
We now give the axiom schemas that describe the behavior of these symbols.
The symbol = is the standard equality predicate defined by the following axioms, in which $x, y, z, x_1, \dots, x_n, y_1, \dots, y_n$ are distinct variables: 

\begin{align}
    & \forall x.\ x=x \label{ax:firsteq} \tagqianaaxiom \\
    & \forall x,y.\ x=y \rightarrow  y=x \tagqianaaxiom \\
    & \forall x,y,z.\ x=y \land y=z \rightarrow x=z \tagqianaaxiom \\
    & \label{penu} \forall x_1, .., x_n, y_1, .., y_n.\ x_1=y_1 \land .. \land x_n=y_n \rightarrow f(x_1, .., x_n) = f(y_1, .., y_n) \tagqianaaxiom \\
    & \forall x_1, .., x_n, y_1, .., y_n.\ x_1=y_1 \land .. \land x_n=y_n \rightarrow p(x_1, .., x_n) \leftrightarrow p(y_1, .., y_n) \tagqianaaxiom  \label{ax:lasteq}
\end{align}
\noindent \textbf{Range:} $f \in F, p \in P $ \\

\noindent The symbol $\reach$ checks whether its argument can be expressed as a term. More precisely, $\reach(t)$ is true if $t$ is a closed term or has all its open variables behind a $\quot$ statement:
\begin{align}
    & \forall x.\ \reach(\quot(x)) \label{reach1} \tagqianaaxiom \\
    & \forall t_1, \dots, t_n.\ (\reach(t_1) \land \dots \land \reach(t_n)) \label{reach3} \rightarrow \reach(f(t_1, \dots, t_n)) \tagqianaaxiom
\end{align}
\noindent\textbf{Range:} $f \in F$ \\

\noindent The symbol $\wft$ stands for ``well-formed term''. Intuitively,  $\wft(t)$ is true iff $t$ represents a quotation of a well-formed term (i.e.,~$t \in \tqv$): 
\begin{align}
    & \forall y.\ \wft(\quot(y)) \label{wft1} \tagqianaaxiom \\
    & \wft(\fquo{x}) \hspace{1cm} \tagqianaaxiom \\
    & \forall t_1, \dots, t_n.\ (\wft(t_1) \land \dots \land \wft(t_n)) \rightarrow \wft(\bm f(t_1, \dots, t_n)) \label{wft4} \tagqianaaxiom
\end{align}
\noindent\textbf{Range:} $\fquo{x} \in \cv, \bm f \in \fq$ \\

\leavevmode \noindent The symbol $\E$ is the counterpart to $\uq$ on quoted terms. More precisely, $\E(t)$ is defined to inductively evaluate to the value that $t$ is a quotation of, when applicable: 
\begin{align}
    & \forall t.\ E(\quot(t)) = t \label{ax:E2} \tagqianaaxiom \\
    & \label{ax:E3} \forall t_1, .., t_n.\ \reach(t_1) \land .. \land \reach(t_n) \rightarrow E(\bm f(t_1, .., t_n)) = f(E(t_1), .., E(t_n))  \tagqianaaxiom \\
    & \forall t_1, .., t_n. (\reach(t_1) \land .. \land \reach(t_n)) \rightarrow  E(\bm p(t_1, .., t_n)) = \bm p(t_1, .., t_n) \tagqianaaxiom \\
    & \forall t_1, t_2.\ E(\bland(t_1,t_2)) = \bland(t_1,t_2) \label{ax:E4} \tagqianaaxiom \\
    & \forall t_1, t_2.\ E(\bforall(t_1,t_2)) = \bforall(t_1,t_2) \label{ax:E5} \tagqianaaxiom \\
    & \forall t.\ E(\bneg(t)) = \bneg(t) \tagqianaaxiom \label{ax:E6} \\
    & E(\fquo{x}) = \fquo{x} \label{ax:E7} \tagqianaaxiom
\end{align}
\noindent \textbf{Range:} $\fquo{x} \in \cv$, $f \in F_b, p \in P$ \\

\noindent When no step of this induction can be carried out, we have $E(t) = t$ (Axiom Schema A24-A28).


The symbol $\sub$ is an in-logic counterpart of the substitution operator in Definition~\ref{def:substitutionquotations}. The term $\sub(t_1, t_2, t_3)$ represents $t_1 [t_2 \leftarrow t_3]_q$: 
\begin{align}
    & \forall t.\ \reach(t) \rightarrow \sub(\fquo{x}, \fquo{x}, t) = t \label{ax:sub1} \tagqianaaxiom \\
    & \forall t.\ \reach(t) \rightarrow \sub(\fquo{x}, \fquo{y}, t) = \fquo{x} \tagqianaaxiom \label{ax:sub2} \\
    & \forall t,t_1, .., t_n.\ (\reach(t_1) \land .. \land \reach(t_n))  \rightarrow  \nonumber \\ & \hspace*{4cm} \sub(\bm f(t_1, .., t_n), \fquo{x}, t) = 
    \bm f(\sub(t_1,\fquo{x},t), .., \sub(t_n,\fquo{x},t)) \label{ax:sub4} \tagqianaaxiom \\
    & \forall t_1,t_2.\ (\reach(t_1) \land \reach(t_2)) \rightarrow \sub(\bforall(\fquo{x}, t_1), \fquo{x}, t_2) = \bforall(\fquo{x}, t_1) \label{ax:sub5} \tagqianaaxiom \\
    & \forall t_1,t_2.\ (\reach(t_1) \land \reach(t_2)) \rightarrow \sub(\bforall(\fquo{y}, t_1), \fquo{x}, t_2) = \bforall(\fquo{y}, \sub(t_1, \fquo{x}, t_2)) \label{ax:sub6} \tagqianaaxiom \\
    & \forall t_1, t_2.\ (\reach(t_1) \land \reach(t_2)) \rightarrow  \sub(\quot(t_1), \fquo{x},t_2) = \quot(t_1) \label{ax:sub7}
    \tagqianaaxiom
\end{align}
\noindent \textbf{Range:} $\fquo{x}, \fquo{y} \in \cv$ ($\fquo{x} \neq \fquo{y}$)$, p \in P, f \in F \setminus \{\bforall, \quot\}$ \\

\noindent We group all these helper axiom schemas together as a theory $\hof$:
\begin{definition}
    $\hof := \text{\ref{ax:firsteq}-\ref{ax:sub7}}$.
\end{definition}

\subsection{Finite Counterparts to Infinite Schemas} \label{subsec:infiToFinFiniteAxio}

Let us write $\htr = \AoneAfour$.
The set $\htr$ defines the behavior of $\truth$ on well-formed formula quotations. 
Now that we have introduced new symbols to act as in-logic counterparts to the most important meta-operators of these schemas, we can introduce new finite schemas that mimic the behavior of $\AoneAfour$ with finitely many formulas.
This is done with $\htf$: 

\begin{definition} \label{def:htf}
    $\htf = \text{\ref{ax:T1}-\ref{ax:T4}}$.
  \begin{align}
      & \label{ax:T1}  \forall t_1, .., t_n.\ (\wft(t_1) \land .. \land \wft(t_n)) \rightarrow  \truth(\bm p(t_1,.., t_n)) \leftrightarrow p(E(t_1), .., E(t_n)) \tag{A1\expofinite} \\
      & \forall t_1, t_2.\ (\reach(t_1) \land \reach(t_2))  \rightarrow  \truth(t_1 \bland t_2) \leftrightarrow (\truth(t_1) \land \truth(t_2)) \label{ax:T2} \tag{A2\expofinite} \\
      & \forall t_1.\ \reach(t_1) \rightarrow \truth(\bneg t_1) \leftrightarrow (\neg \truth(t_1)) \label{ax:T3} \tag{A3\expofinite} \\
      & \forall t_1.\ \reach(t_1) \rightarrow \label{ax:T4}\tag{A4\expofinite} \truth(\bforall(\fquo{x},t_1)) \leftrightarrow (\forall x.\ \truth(\sub(t_1,\fquo{x},\quot(x))))
  \end{align}
\noindent  \textbf{Range:} $p \in P \setminus \{\truth\}, x \in V$
\end{definition}

\noindent Likewise, we define schema~\ref{ax:istForallFinite} as a finite counterpart to schema~\ref{as:oldDefss}.
\begin{align}
    & \forall t_1, t_2.\ \reach(t_1) \rightarrow \textit{ist}(t_2, \bforall(\fquo{x},t_1)) \rightarrow  \forall x.\ \textit{ist}(t_2, \sub(t_1,\fquo{x},\quot(x))) \tag{A11\expofinite}\label{ax:istForallFinite}
\end{align}

\noindent\textbf{Range:} $x \in V$

\subsection{Correctness} \label{subsec:propsFiniteAxio} 

\begin{definition} \label{def:hf}
    We can now formally define the finite axiomatization of Qiana on $H$ as the set $\hf$:
    \begin{align*}        
        \hf & := H \cup \hqf \cup \hof \cup \htf \cup \ref{ax:istForallFinite}
    \end{align*}
\end{definition}

\noindent Recall that the Qiana-closure of $H$ is $\hq = H \cup \hqf \cup \AoneAfour \cup \ref{as:oldDefss}$.
The following theorem and corollary say that it is equivalent to reason with the theories $\hq$ or $\hf$:
\begin{theorem} \label{th:mainThProveFiniteAxio}
    $\hq$ is coherent if and only if $\hf$ is coherent.
\end{theorem}
\begin{proof}
    Recall that ``$\hq$ is coherent'' is equivalent to $\hq \not\models \bot$. Hence the theorem becomes $\hq \not\models \bot$ iff $\hf \not\models \bot $.
    In the supplementary material, we prove both directions of this equivalence in Propositions~\ref{prop:fiFinGiveBase} and \ref{prop:fiBaseFiveFin}. 
\end{proof}

\begin{corollary}
    $\hq \models \varphi$ if and only if $\hf \models \varphi$, for all $\varphi$
\end{corollary}
\begin{proof}
    Apply Theorem \ref{th:mainThProveFiniteAxio} to the theory $H \cup \{\neg \varphi\}$.
\end{proof}

\noindent The following proposition says that the number of formulas created by the finite axiomatization process is quadratic in the total number of symbols, excluding the variables that are not quotable.

\begin{proposition}
  The cardinal of $\hqf \cup \hof \cup \htf$
  is in $\mathcal{O}(|S|^2)$, where
  $|S|$ is the total number of symbols in $S$, excluding $V_\infty \setminus V$.
\end{proposition}

%
%

We introduced a finite number of quoted variables to make this finite axiomatization process possible.
Indeed, each quoted variable needs to appear at least once within the finite axiomatization. Since the process uses finitely many formulas of finite length, it cannot handle an infinite number of quoted variables.
Nevertheless, any reasoning that can be carried out with an infinite number of variables can also be done with a finite number of variables:

\begin{proposition}
    Let $H$ be a theory, and let $\hq^n$ be the Qiana closure of $H$ where $V$ has size $n \in \mathbb{N}$, and let $\hcinf$ be the Qiana closure of $H$ obtained by allowing the set $V$ to be infinite. 
    Let $\varphi$ be any well-formed closed formula.
    Then if $\hcinf \models \varphi$ then there is some $n \in \mathbb{N}$ such that $\hq^n \models \varphi$.
\end{proposition}
\begin{proof}
    Any proof derivation of $\varphi$ from $\hcinf$ uses finitely many formulas, which are all in $\hq^n$ for some $n$.
\end{proof}
\noindent Thus, the finiteness of the set $V$ of quotable variables is not a limitation on the reasoning power. 
When checking some entailment, we can iteratively increase the size of $V$ to check if the entailment appears.
Considering that our theory is semi-decidable rather than decidable, we do not lose any deductive power.

    \subsection{Using Qiana in Theorem Provers} \label{sec:TP}
Our finite axiomatization allows us to 
transform the Qiana closure of any finite theory into an equisatisfiable finite first-order logic theory, which can then be fed into an Automated Theorem Prover (ATP). 
We have implemented a translator (in Python) that accepts a set of Qiana formulas, derives their signature, and outputs a finite set of Qiana axioms in the TPTP syntax~\cite{TPTP}.

This allows us to run the Romeo and Juliet example from Section~\ref{sec:Example} in the Vampire theorem prover~\cite{VampireMainarticle}. 
The reasoning takes 0.05 seconds on an 8th-generation Intel CPU laptop. Vampire duly proves that both Romeo and Juliet die.
The code and the example are available at \url{https://github.com/dig-team/Qiana}. 



\section{Temporality in Qiana} \label{sec:Time}
    In this section, we extend Qiana for temporal reasoning.
There are several modal logics to deal with time, with the most important ones being Event Calculus (EC), linear-temporal logic (LTL), and computational tree logic (CTL). For our work, we chose EC for three reasons. First, time instants are explicit in EC (and not in LTL or CTL). Second, EC is based on first-order logic and thus more amenable to a translation to Qiana. Finally, time instants and intervals play a similar role in event calculus as contexts in Qiana, which makes the adaptation of event calculus to Qiana relatively straightforward.

\subsection{Overview of Event Calculus}
Event calculus is a popular family of formalisms to represent actions and their effects on systems through time.
In this article, we follow the definitions of~\cite{EventCalculusExplained}, which we chose for their clarity and conciseness. 
Event calculus is based on the following concepts:
\textit{Fluents} are properties of the system that can change over time; there is typically a finite set of fluents under consideration.
\textit{Actions} (also called \textit{Events}) occur at points in time or during time intervals and can change the value of fluents.

Here are a few example sentences written in classical event calculus and toying with the Romeo and Juliet story:
\begin{itemize}
    \item The construction $\textit{HoldsAt}(\textit{Alive\_Romeo},t)$, meaning that Romeo is alive at instant $t$. Here, the fluent is $\textit{Alive\_Romeo}$, and the time is $t$.
    \item The construction
    $\textit{Happens}(\textit{Drink\_Potion\_Juliet},t_1,t_2)$
    means that Juliet drinks a potion between time $t_1$ and $t_2$. Here, the action is again $\textit{Drink\_Potion\_Juliet}$. 
    \item The construction $\textit{Happens}(\textit{Drink\_Potion\_Juliet},t)$ means that Juliet drinks a potion at time $t$. Here the action is $\textit{Drink\_Potion\_Juliet}$.  
    Note the operator overload on $\textit{Happens}$ with the previous statement. In fact, $\textit{Happens}(a,t)$ is simply syntactic sugar for $\textit{Happens}(a,t,t)$.
\item The construction $\textit{Initiates}(\textit{Appear\_Dead\_Juliet},\textit{Drink\_Potion\_Juliet},t)$  means that drinking a potion makes Juliet start appearing dead at time $t$. Here the fluent is $\textit{Appear\_Dead\_Juliet}$ and the action is $\textit{Drink\_Potion\_Juliet}$.
\end{itemize}

It should be noted that fluents are atomic and cannot be connected to form more complex fluents (there are no equivalents of $\neg$ or $\land$ on fluents).
Also, this formalism features a notion of inertia. When a fluent becomes true it remains so unless it is ``clipped'', which is represented by a dedicated predicate $\textit{Clipped}$. 
This behavior is formally defined by axioms \ref{ax:EC1}, \ref{ax:EC2}, and \ref{ax:EC3} in Subsection~\ref{subsec:ECQiana}, which we present with Qiana notations.
Table~\ref{table:ECops} lists and describes the main operators of event calculus.

\begin{table}[h!]
    \caption{Matching Event Calculus operators to their descriptions} \label{table:ECops}

    \centering
    \scalebox{0.9}{
    \begin{tabular}{|c|p{12cm}|} 
    \hline
    Event Calculus operator & Description \\ 
    \hline
    0 & The first time instant \\
    $i_1 < i_2$ & Instant $i_1$ occurs before instant $i_2$ \\
    $\textit{Initially}(\bm \varphi)$ & Nontemporal quoted formula $\bm \varphi$ holds at the beginning \\
    $\textit{HoldsAt}(\bm \varphi, i_1)$ & Nontemporal quoted formula $\bm \varphi$ holds at time $i_1$ \\ 
    $\textit{Happens}(a, i_1, i_2)$ & Action $a$ happens between times $i_1$ and $i_2$ \\
    $\textit{Initiates}(a, \bm \varphi, i_1)$ & If action occurs at time $i_1$ it initiates $\bm \varphi$ at that time \\
    $\textit{Terminates}(a, \bm \varphi, i_1)$ & If action occurs at time $i_1$ it terminates $\bm \varphi$ at that time \\
    $\textit{Releases}(a, \bm \varphi, i_1)$ & $\bm \varphi$ is not subject to inertia after action $a$ at time $i_1$ \\
    $\textit{Clipped}(i_1, \bm \varphi, i_2)$ & $\bm \varphi$ is terminated between times $i_1$ and $i_2$ \\
    $\textit{Declipped}(i_1, \bm \varphi, i_2)$ & $\bm \varphi$ is initiated between times $i_1$ and $i_2$ \\
    \hline
    \end{tabular}}
\end{table}
This concludes our presentation of event calculus. 
In the next subsection, we describe how to adapt event calculus to Qiana.

\subsection{Event calculus in Qiana} \label{subsec:ECQiana}

Because the full event calculus of~\cite{EventCalculusExplained} is based on first-order logic, its adaptation to Qiana will be relatively straightforward.
We allow using any quoted formula as a fluent and write the axioms of event calculus in Qiana.
This requires the introduction of a few new symbols to Qiana to match the operators of event calculus.

The event calculus presented in~\cite{EventCalculusExplained} also contains an operator $\textit{Initially}_N$ to say that some fluent is initially false. Thanks to the quote symbol $\bneg$, we write $\textit{Initially}(\bneg\ f)$ instead of $\textit{Initially}_N(f)$.
This makes use of the power of negation over quoted formulas of Qiana, which is not available in the original event calculus.


Here are the axioms of event calculus in Qiana, adapted from ~\cite{EventCalculusExplained}:

\begin{align}
    &\forall f, t .\ \textit{HoldsAt}(f, t) \leftarrow \textit{Initially}_P(f) \land \neg \textit{Clipped}(0, f, t) \label{ax:EC1}\tag{EC1} \\
    &\forall f, t_3, a, t_1, t_2 .\ \textit{HoldsAt}(f, t_3) \notag \\ & \hspace*{1cm} \leftarrow \textit{Happens}(a, t_1, t_2) \land \textit{Initiates}(a, f, t_1) \land \neg \textit{Clipped}(t_1, f, t_3) \land t_2 < t_3 \label{ax:EC2}\tag{EC2} \\
    &\forall t_1, f, t_4 .\ \textit{Clipped}(t_1, f, t_4) \leftrightarrow \exists a, t_2, t_3 .\ \textit{Happens}(a, t_2, t_3) \land \notag \\ &\hspace*{1cm} (\textit{Terminates}(a, f, t_2) \lor \textit{Releases}(a, f, t_2)) \land t_1 < t_3 \land t_2 < t_4 \label{ax:EC3}\tag{EC3} \\
    &\forall f, t .\ \neg \textit{HoldsAt}(f, t) \leftarrow \textit{Initially}(\bneg f) \land \neg \textit{Declipped}(0, f, t)\tag{EC4} \\
    &\forall f, t_3, a, t_1, t_2 .\ \neg \textit{HoldsAt}(f, t_3) \notag \\ & \hspace*{1cm} \leftarrow \textit{Happens}(a, t_1, t_2) \land \textit{Terminates}(a, f, t_1) \land \neg \textit{Declipped}(t_1, f, t_3) \land t_2 < t_3 \tag{EC5} \\
    &\forall t_1, f, t_4 .\ \textit{Declipped}(t_1, f, t_4) \leftrightarrow \exists a, t_2, t_3 .\ \notag \\ &\hspace*{1cm} \textit{Happens}(a, t_2, t_3) \land (\textit{Initiates}(a, f, t_2) \lor \textit{Releases}(a, f, t_2)) \land t_1 < t_3 \land t_2 < t_4 \tag{EC6} \\
    &\forall a, t_1, t_2 .\ \textit{Happens}(a, t_1, t_2) \rightarrow t_1 \leq t_2\label{ax:EC7}\tag{EC7}
\end{align}


\begin{remark}
    We do nothing to prevent the use of the quotations of temporal statements as fluents. 
    For example, nothing explicitly prevents the use of $\fquo{\textit{Happens}(\textit{drink}(P,J),t)}$ as a fluent. 
    This is a choice we make to simplify the formalism, but we also make no special effort to give them a special meaning.
    Hence, they can be considered as any other meaningless quotation we could, in theory, pass to a temporal operator.
    Since no axiom allows their introduction in a temporal context, this creates no problem.
\end{remark}

\subsection{Example}

We will now adapt our running example of Romeo and Juliet (Section~\ref{subsec:suicideExample}) to our temporal framework.
We will tell the same story as before, but we will now account for time:
At first, Romeo is alive, then he sees Juliet, and then he dies. We will use the following axioms:
\begin{align}
    & \forall \varphi.\ \ist(\textit{says}(\textit{L}), \varphi) \rightarrow \truth(\varphi) \label{ex:EC_RJ1}\\
    & \ist(\textit{says(\textit{L})},  \bmnested{\forall t.\ \textit{Happens}(\textit{drink}(P,J),t) \rightarrow \textit{Initiate}(\textit{drink}(J,P), \bmnested{\textit{LookDead}}(J), t)}) \label{ex:EC_RJ2}\\
    & t_\textit{drink} < t_\textit{see} \label{ex:EC_RJ8}\\
    & \textit{Happens}(\textit{drink}(P,J), t_\textit{drink}) \label{ex:EC_RJ3}\\
    & \textit{Happens}(\textit{see}(R,J), t_\textit{see}) \label{ex:EC_RJ6} \\
    & \neg \textit{Clipped}(\textit{LookDead}(J), t_\textit{drink}, t_\textit{see}) \label{ex:EC_RJ4}\\
    & \forall t_1, p.\ \textit{HoldsAt}(\bmnested{\textit{LookDead}}(\quot(p)), t_1) \rightarrow  \textit{Initiate}(\textit{See}(R,p), \bmnested{\ist(\textit{bel}(R), \bmnested{dead(p)})}, t_1) \label{ex:EC_RJ5} \\
    & \forall t_1.\ \textit{HoldsAt}(\ist(\textit{bel}(R), \fquo{dead(J)}), t_1) \rightarrow \textit{Happens}(\textit{die}(R), t_1) \label{ex:EC_RJ7}
\end{align}

Note the different kinds of elements. 
The terms $\textit{see}(R,J)$ and $\textit{die}(R)$ are actions.
$\textit{LookDead}(J)$ is a formula (like $\textit{dead}(J)$), its quotation $\fquo{\textit{LookDead}(J)}$ is a fluent.
Two points in time are important to the story: the time $t_\textit{drink}$ when Juliet drinks the potion and the time $t_\textit{see}$ when Romeo sees Juliet. 
Formulas~\ref{ex:EC_RJ1} and~\ref{ex:EC_RJ2} behave similarly to their counterparts Formula~\ref{eq00} and Formula~\ref{eq3} in Section~\ref{subsec:suicideExample}, 
except that Laurence now says Juliet will ``start looking dead'', rather than simply ``look dead''.
Formula~\ref{ex:EC_RJ3} states that Juliet drinks the potion, and formula~\ref{ex:EC_RJ4} states that Juliet does not wake up before $t_\textit{see}$.
Hence, Juliet looks dead at time $t_\textit{see}$.
Together with formulas~\ref{ex:EC_RJ6} and~\ref{ex:EC_RJ5} this tells us that Romeo starts believing Juliet is dead, which leads to his own death (Formula~\ref{ex:EC_RJ7}).

\subsection{The Frame Problem and Differences Between Qiana and Event Calculus}

The frame problem is a classic issue in formalisms that model change over time, and it is the reason for a variation in the semantics between the event calculus of \cite{EventCalculusExplained} and Qiana. \\

To explain this difference, we begin by explaining the frame problem.
The frame problem arises because things generally remain unchanged unless something happens to them to make them change. This is referred to as inertia—the natural tendency of systems to stay in their current state. The challenge is how to formalize this idea without having to explicitly list every possible situation in which something might change. 

The solution used by event calculus is circumscription. Entailment in event calculus relies on axioms, normal first-order logic (FOL) entailment, and a function called \emph{circumscription}.
The idea is to use $\textit{Circum}(\Gamma)$, which represents an extension of $\Gamma$ with reasonable assumptions. 
The full definition of $\textit{Circum}(\Gamma)$ is omitted here and we refer the reader to \cite{EventCalculusExplained} for a more complete explanation and additional examples.
The two following facts are important:
\[
\textit{Circum}(\Gamma) \models \Gamma
\]
\[
\Gamma \models_\textit{EC} \varphi \Leftrightarrow \textit{Circum}(\Gamma) \cup \textit{EC} \models \varphi
\]

\noindent where $\models_\textit{EC}$ is the entailment of event calculus and \textit{EC} the set of axioms of event calculus.

In a sense, the entailment of event calculus (with circumscription) is weaker than normal FOL entailment. Valid FOL entailment still holds under it, but the entailment of event calculus also makes additional assumptions.

In Qiana, we import the axioms of event calculus, but we do not use this notion of entailment.
Instead, we rely on normal first-order entailment.
Hence, we do not import event calculus's solution to the frame problem. Axiom~\ref{ax:EC1} states that fluents stay fixed unless they are clipped, but nothing introduces a notion of inertia on the ``clipped'' predicate itself.
It would certainly be possible to adapt Qiana's entailment to include the circumscription mechanism.
But it is an important feature of Qiana that its entailment is simply FOL entailment with some axioms, allowing compatibility with existing first-order theorem provers.
Our form of entailment is stronger than the one of event calculus, meaning it is sound but not complete with respect to it.

We prove this with Proposition~\ref{prop:ECQvsEC}.

\begin{proposition} \label{prop:ECQvsEC}
    Let $\Gamma$ be a set of formulas valid under event calculus. Let $H_\textit{TQ}$ be the set of axioms of temporal Qiana, and let $\varphi$ be a temporal formula valid under event calculus. Then:
    \[
    \Gamma \cup H_\textit{TQ} \models \varphi \implies \Gamma \models_\textit{EC} \varphi
    \]
\end{proposition}

\begin{proof}
    Suppose $\Gamma \cup H_\textit{TQ} \models \varphi$. Then, we have:
    \[
    \textit{Circum}(\Gamma_\textit{ECQ}) \cup H_\textit{TQ} \models \varphi
    \]
    Since the temporal axioms of Qiana follow the axioms of event calculus, we obtain:
    \[
    \textit{Circum}(\Gamma_\textit{ECQ}) \cup H_\textit{TQ} \models \varphi \implies \textit{Circum}(\Gamma_\textit{ECQ}) \cup \textit{EC} \models \varphi
    \]
    By definition, this is equivalent to:
    \[
    \Gamma \models_\textit{EC} \varphi
    \]
\end{proof}

The temporal version of Qiana follows event calculus but does not include its additional assumptions regarding inertia. 
Considering that temporality and contexts are largely orthogonal in Qiana, nothing stops us from defining a circumscription operator in Qiana and using it to define another notion of entailment for our logic.
However, the fact that entailment in Qiana is just FOL entailment with certain axioms is an important feature of the logic.
Therefore, we consider the simple application of the axioms of event calculus to be the most appropriate way to model temporal reasoning in Qiana.
\fms{Can we end on something positive here? Like anything done in EC can be done in Qiana? Or some specific subset/superset of EC can be done in Qiana?}
\SC{How about this?}

\section{Typing Qiana} \label{sec:Types}
    A Qiana theory uses different types of objects: formulas, quoted formulas, terms, and (together with event calculus) events, actions, and fluents. The boundary between these types is sometimes porous: 
For example, the following formula, taken from the example in Section~\ref{subsec:suicideExample}, implies that all objects are contexts. This is not a real problem in this case, but it is inelegant.
\begin{align*}
    & \forall c, x.\ \textit{appearDead}(x) \ \rightarrow{} \ \textit{ist}(c, \quo{appearDead}(x))
\end{align*}

\fms{give example here where an object is used in a bad way and it's not prohibited in Qiana}
\SC{Done with a real example.}
One way to make the distinction clearer is to resort to typed first-order logic (also known as many-sorted FOL). In what follows, we define a typed version of Qiana that distinguishes different types of objects more clearly. 
This typed version of Qiana can be considered more intuitive, and the ability to distinguish natively between different types could be useful in concrete applications with complex rules mixing different types of objects.
However, types make formulas longer and create multiple technical issues that have to be handled throughout the entire process of typing Qiana. 
In particular, the finite axiomatization process of Section~\ref{sec:FiniteAxio} becomes more complex and much longer (see Subsection~\ref{sec:Types:subsec:FiniteAxio}). 
Hence, we define the typed version of Qiana here merely as a theoretical exercise. 

\subsection{A Quick Summary of Typed First-Order Logic}
Typed FOL (first-order logic) is sometimes also called many-sorted FOL.
We will give only a quick summary of the topic here and redirect the unfamiliar reader to the literature~\cite{BookManySorted}. 
Many variations on typed FOL have been proposed, but we will use the basic one simply called ``many sorted FOL'', which is what we present below.

In many-sorted FOL, the signature includes a finite set $\types$ of types.
All function symbols and predicate symbols have a signature, indicating the types of their arguments and the type of the output in the case of functions.
Variable symbols also have an associated type and we assume there are countably many symbols of each type. 
The type of a term is given by its top-level symbol; if that symbol is a function symbol, it is the symbol's output type, and if it is a variable symbol, it is its type.
A term can be used as an argument to a predicate or function only if it has the correct type, as indicated by the signature of the predicate or function.
Models of many-sorted FOL are similar to models of unsorted FOL, but the domain of the model is partitioned into disjoint sets, one for each type. 

\begin{remark}
    Ideally, we would have preferred to use a flavor of typed FOL that allows non-disjoint types, such as order-sorted FOL.
    However, as far as we know, all ATPs (Automated Theorem Provers) that support TPTP input (and, in fact, all ATPs we know of) support only disjoint types.
    In order to maintain the compatibility of Qiana with existing ATPs, we will have to use disjoint types.
\end{remark}

\subsection{Types for Qiana} \label{sec:Types:subsec:TypingQiana}
Table~\ref{tab:types} introduces the types used to produce typed-Qiana, along with their short descriptions.

\begin{table}[h]
    \caption{Types and their descriptions in Qiana} \label{tab:types}

    \centering
    \begin{tabular}{|c|c|c|}
        \hline
        \textbf{Type} & \textbf{Name} & \textbf{Description} \\
        \hline
        $o$ & \textit{Objects} & Non-Qiana objects (people, places, ...)\\
        \hline
        $q$ & \textit{Quotations} & Type of all quotations (quoted formulas, quoted terms, ...) \\
        \hline
        $c$ & \textit{Contexts} & Contexts\\
        \hline
        $\tau$ & \textit{Time instant} & Event calculus points in time.\\
        \hline
        $a$ & \textit{Actions} & Event calculus actions \\
        \hline
   \end{tabular}
\end{table}

We extend the FOL signature $(F,P,V_\infty,\delta)$ to $(F,P,V_\infty,\delta,\types)$ by adding a finite set of types $\types$, and adapting $\delta$ and $V_\infty$.
We have $U = \{o, q, c, \tau, a\}$.

\noindent Whereas $\delta$ gave the arity of symbols so far, now $\delta$ takes the types into account. More precisely, $\delta$ gives a signature to each predicate and function symbol, and a type to each variable symbol.
\begin{align*}
    & \text{for each } p \in P,\ \delta(p) \in U^\star \\
    & \text{for each } f \in F,\ \delta(f) \in U^\star \times U \\
    & \text{for each } v \in V_\infty,\ \delta(v) \in U
\end{align*}

where $\delta(f) \in U^\star \times U$ means that $f$ has a signature of the form $\gamma_1 \times \dots \times \gamma_n \rightarrow \gamma_o$, with $n$ the arity of $f$ and $\gamma_1, \dots, \gamma_n, \gamma_o \in U$. For example $\delta(\ist) = c \times q$. \\

\noindent Also, we assume that there is an infinity of variable symbols of each type:
\[\text{for each } \gamma \in \types,\ |\set{v \in V_\infty | \delta(v) = \gamma}| = +\infty\]

\begin{remark}
    If needed in practical applications, we can always split the type $o$ into multiple subtypes.
    Here, we present everything with a single type for non-Qiana-specific objects; this is without loss of generality, and everything can be straightforwardly adapted to the case where $o$ is split into many subtypes.
\end{remark}

We introduce multiple notations to indicate the types of symbols we use.
We will alternate between the notations depending on which is most convenient for the formulas we write. 

First, we can directly indicate the type of one or multiple symbols in a separate line:
The following line indicates that $f$ takes as arguments an object and a time instant, and returns a context, that $p$ takes a context, and that $p_2$ has no argument (i.e., it is an atomic predicate).
\[f: o \times \tau \rightarrow c; p: c; p_2 : ()\]

We can also indicate the type of a symbol by writing the type as an exponent.
In the following formula, $t_1$ and $t_2$ are variables of type $o$, $t_3$ is a variable of type $c$, and $p$ can implicitly be deduced to be a predicate of signature $o^2 \times c$.
\[\forall t_1^o, t_2^o, t_3^c.\ p(t_1^o, t_2^o, t_3^c)\]

We can indicate the type of variables during quantification. The following formula is equivalent to the previous one:
\[\forall^o t_1, t_2.\ \forall^c t_3.\ p(t_1, t_2, t_3)\]

Lastly, in some contexts, the type can be clearly inferred from the context, and therefore, no additional type annotations are necessary. \\

In this section, we will first always include a type indication. This is because the first axioms we present will also serve as examples of our typing notations.
However, we will gradually omit type annotations as they would only make our formulas less readable without any gain in clarity.

\begin{example}
    In the following, we explicitly type each symbol as an example. In total, this formula contains terms of type $q$, $c$, $\tau$, and $a$.
    \begin{align*}
        \textit{between}:&\ \tau \times \tau \rightarrow c \\
        \varphi :&\ () \rightarrow q \\
        a :&\ () \rightarrow a
    \end{align*}
    \[\forall t_1^\tau, t_2^\tau.\ \ist(\textit{between}(t_1,t_2), \textit{Initiates}(a,\varphi)) \rightarrow \ist(\textit{between}(t_1, t_2), \varphi)\]

\end{example}
\begin{example}
    \noindent We can also type the example formula below from Example~\ref{ex:allCapuletsAreNice}.
    \begin{align*}
        \textit{believes} :&\ o \rightarrow c \\
        \textit{Cap} :&\ o \\
        \textit{nice} :&\ o \\
        \textit{ist} :&\ c \times q
    \end{align*}
    \[\textit{ist}(\textit{believes}(J), \fquo{\forall x.\ \textit{Cap}(x) \rightarrow \textit{nice}(x)}) \rightarrow \forall x^o.\ \textit{ist}(\textit{believes}(J), \fquo{\textit{Cap}}(\quot_o (x^o)) \fquo{\rightarrow}\ \fquo{\textit{nice}}(\quot_o(x^o)))\]
\end{example}

\subsection{Typing the General Qiana Axioms}

We now present a typed version of the general Qiana axioms presented in Section~\ref{sec:axioms}, using the types introduced in Table~\ref{tab:types}. 

\begin{align}
    & \forall x_1^{\gamma_1}, \dots, x_n^{\gamma_n}.\ \truth(\bm p(t_1,\dots, t_m)) \leftrightarrow  p(\uq(t_1), \dots, \uq(t_m)) \label{ax:type:type:str1} \tagqianaaxiom
\end{align}
\textbf{Range:} $p \in P, t_1, .., t_n \in \tqtv \text{ with } \forall i.\ \delta(\uq(t_i)) = \delta(p)_i, x_1^{\gamma_1}, .., x_n^{\gamma_n} \text{ the free variables in } t_1, .., t_m$

\begin{align}
    & \forall x_1^{\gamma_1} \dots, x_n^{\gamma_n}\ \truth(A \bland B) \leftrightarrow (\truth(A) \land \truth(B)) \label{ax:type:str2} \tagqianaaxiom \\
    & \forall x_1^{\gamma_1} \dots, x_n^{\gamma_n}\ \truth(\bneg A) \leftrightarrow (\neg \truth(A)) \label{ax:type:str3} \tagqianaaxiom \\
    & \forall x_1^{\gamma_1} \dots, x_n^{\gamma_n}\ \truth(\bforall(\fquo{x} ,A)) \leftrightarrow  (\forall^\gamma x.\ \truth(A[\fquo{x} \leftarrow \quot_\gamma(x)]_q)) \label{ax:type:str4} \tagqianaaxiom
\end{align}
\textbf{Range:} $A, B \in \tqv, x_1^{\gamma_1} \dots, x_n^{\gamma_n} \text{ the free variables in } A, B, \bm x \text{ a quoted variable of type } \gamma$

\begin{align}
    & \forall^c x_c, \forall^q x_1, x_2.\ \textit{ist}(x_c,x_1 \bland x_2) \rightarrow \textit{ist}(x_c, x_1) \label{ax:type:qiana1} \tagqianaaxiom\\
    & \forall^c x_c, \forall^q x_1, x_2.\ \textit{ist}(x_c, {x_1 \bland x_2}) \leftrightarrow \textit{ist}(x_c,  {x_2 \bland x_1}) \label{ax:type:qiana2}  \tagqianaaxiom \\
    & \forall^c x_c, \forall^q x_1.\ \textit{ist}(x_c, {\bneg \bneg x_1}) \leftrightarrow \textit{ist}(x_c, {x_1}) \label{ax:type:qiana3}  \tagqianaaxiom \\
    & \forall^c x_c, \forall^q x_1, x_2, x_3.\ \textit{ist}(x_c, {(x_1 \bland x_2) \bland x_3})\leftrightarrow \textit{ist}(x_c, {x_1 \bland (x_2 \bland x_3)}) \label{ax:type:qiana4}  \tagqianaaxiom\\
    & \forall^c x_c, \forall^q x_1, x_2, x_3.\ \textit{ist}(x_c, {(x_1 \bland x_2) \blor x_3}) \leftrightarrow \textit{ist}(x_c, {(x_1 \blor x_3) \bland (x_2 \blor x_3)}) \label{ax:type:qiana5}  \tagqianaaxiom \\
    & \forall^c x_c, \forall^q x_1, x_2.\ \textit{ist}(x_c, x_1 \blor x_2) \land \textit{ist}(x_c, {\bm \neg} x_1) \rightarrow \textit{ist}(x_c,x_2) \label{ax:type:istMP} \tagqianaaxiom \\
    & \forall^c x_c.\ \textit{ist}(x_c, \bforall(\fquo{x},\bm{\varphi})) \rightarrow \forall^\gamma x.\ \textit{ist}(x_c, \bm{\varphi}[\fquo{x} \leftarrow \quot_\gamma(x)]_q) \label{ax:type:oldDefss} \tagqianaaxiom
\end{align}
\textbf{Range:} $\bm x \text{ a quoted variable of type } \gamma, \bforall(\fquo{x}, \bm \varphi) \in \tqf$

\subsection{Typing the Finite Axiomatization} \label{sec:Types:subsec:FiniteAxio}
We now present a typed version of the finite axiomatization presented in Section~\ref{sec:FiniteAxio}, using the types introduced in Table~\ref{tab:types}.
Most are straightforward adaptations of the ones from Section~\ref{sec:FiniteAxio}.
There is a single type $q$ for all quotations, but the core idea of the finite axiomatization process is to go over the quoted formula and interpret them recursively within the top-level domain of discourse.
In particular, the function $E$ matches a quoted term to the actual term it is a quotation of. 
Because terms can have different types, we need to create multiple versions of $E$ with different output types.
We will write $E_{\gamma}$ for the version of $E$ that matches the quotation of a term of type $\gamma$ to said term.

Also, we need to rework the axiomatization of $\wft$ so that it also checks that the quotation is well-typed.
To do so, we introduce one instance of $\wft$ per type $\gamma$, which we write $\wft_{\gamma}$.

Likewise, we introduce multiple versions of the equality predicate = and the reachability predicate $\reach$.

\begin{align}
    & \forall^\gamma x.\ x =_\gamma x \label{type:ax:firsteq} \tagqianaaxiom \\
    & \forall^\gamma x,y.\ x=_\gamma y \rightarrow  y=_\gamma x \tagqianaaxiom \\
    & \forall^\gamma x,y,z.\ x=_\gamma y \land y=_\gamma z \rightarrow x=_\gamma z \tagqianaaxiom \\
    & \label{type:penu} \forall x_1^{\gamma_1}, .., x_n^{\gamma_n}, y_1^{\gamma_1}, .., y_n^{\gamma_n}.\ x_1^{\gamma_1} =_{\gamma_1} y_1^{\gamma_1} \land .. \land x_n^{\gamma_n} = y_n^{\gamma_n} \rightarrow f(x_1^{\gamma_1}, .., x_n^{\gamma_n}) =_{\gamma_o} f(y_1^{\gamma_1}, .., y_n^{\gamma_n}) \tagqianaaxiom \\
    & \forall x_1^{\gamma_1}, .., x_n^{\gamma_n}, y_1^{\gamma_1}, .., y_n^{\gamma_n}.\ x_1^{\gamma_1} =_{\gamma_1} y_1^{\gamma_1} \land .. \land x_n^{\gamma_n} = y_n^{\gamma_n} \rightarrow p(x_1^{\gamma_1}, .., x_n^{\gamma_n}) \leftrightarrow p(y_1^{\gamma_1}, .., y_n^{\gamma_n}) \tagqianaaxiom \label{type:ax:lasteq}
\end{align}
\noindent \textbf{Range:} $f \in F, p \in P, \gamma, \gamma_o, \gamma_1, \dots, \gamma_n \in U, \delta(f) = \gamma_1 \times \dots \times \gamma_n \rightarrow \gamma_o, \delta(p) \gamma_1 \times \dots \times \gamma_n$

\begin{align}
    & \forall^\gamma x.\ \reach_q(\quot_\gamma(x)) \label{type:reach1} \tagqianaaxiom \\
    & \forall t_1, \dots, t_n.\ (\reach_{\gamma_1}(t_1) \land \dots \land \reach_{\gamma_n}(t_n)) \label{type:reach3} \rightarrow \reach_{\gamma_o}(f(t_1, \dots, t_n)) \tagqianaaxiom
\end{align}
\textbf{Range:} $f \in F, \gamma, \gamma_o, \gamma_1, ..., \gamma_n \in U$

\begin{align}
    & \forall y^\gamma.\ \wft_\gamma(\quot_\gamma(y^\gamma)) \label{type:wft1} \tagqianaaxiom \\
    & \wft_\gamma(\fquo{x}) \hspace{1cm} \tagqianaaxiom \\
    & \forall^q t_1, \dots, t_n.\ (\wft_{\gamma_1} (t_1) \land \dots \land \wft_{\gamma_n}(t_n)) \rightarrow \wft_{\gamma_o}(\bm f(t_1, \dots, t_n)) \label{type:wft4} \tagqianaaxiom
\end{align}
\textbf{Range:} $\fquo{x} \in \cv \text{ a quoted variable of type }\gamma, \bm f \in \fq \text{ such that } \delta(f) = \gamma_1 \times ... \times \gamma_n \rightarrow \gamma_o, \gamma \in U$

\begin{align}
    & \forall^\gamma t.\ \reach_\gamma(t) \rightarrow E_\gamma(\quot_\gamma(t)) = t \label{type:ax:E2} \tagqianaaxiom \\
    & \label{type:ax:E3} \forall t_1, .., t_n.\ \reach(t_1) \land .. \land \reach(t_n) \rightarrow E(\bm f(t_1, .., t_n)) = f(E(t_1), .., E(t_n))  \tagqianaaxiom
\end{align}
\noindent \textbf{Range:} $\fquo{x} \in \cv$, $f \in F, p \in P$

\begin{align}
    & \forall t.\ \reach_q(t) \rightarrow \sub(\fquo{x}, \fquo{x}, t) = t \label{type:ax:sub1} \tagqianaaxiom \\
    & \forall t.\ \reach_q(t) \rightarrow \sub(\fquo{x}, \fquo{y}, t) = \fquo{x} \tagqianaaxiom \label{type:ax:sub2} \\
    & \forall t,t_1, \dots, t_n.\ (\reach_q(t_1) \land \dots \land \reach_q(t_n))  \rightarrow  \sub(\bm f(t_1, \dots, t_n),\fquo{x},t) = \notag \\ & \hspace{1cm} \bm f(\sub(t_1,\fquo{x},t), \dots, \sub(t_n,\fquo{x},t)) \label{type:ax:sub4} \tagqianaaxiom \\
    & \forall t_1,t_2.\ (\reach_q(t_1) \land \reach_q(t_2)) \rightarrow \sub(\bforall(\fquo{x}, t_1), \fquo{x}, t_2) = \bforall(\fquo{x}, t_1) \label{type:ax:sub5} \tagqianaaxiom \\
    & \forall t_1,t_2.\ (\reach_q(t_1) \land \reach_q(t_2)) \rightarrow \sub(\bforall(\fquo{y}, t_1), \fquo{x}, t_2) = \bforall(\fquo{y}, \sub(t_1, \fquo{x}, t_2)) \label{type:ax:sub6} \tagqianaaxiom \\
    & \forall t_1^\gamma, t_2^q.\ (\reach_\gamma(t_1) \land \reach_q(t_2)) \rightarrow  \sub(\quot_\gamma(t_1), \fquo{x},t_2) = \quot_\gamma(t_1) \label{type:ax:sub7}
    \tagqianaaxiom
\end{align}
\noindent \textbf{Range:} $\fquo{x}, \fquo{y} \in \cv$ ($\fquo{x} \neq \fquo{y}$)$, f \in F, p \in P, \gamma \in U$

\begin{definition} \label{type:def:htf}
    $\htf = \text{\ref{ax:T1}-\ref{ax:T4}}$
\end{definition}

\begin{align}
    & \label{type:ax:T1}  \forall t_1, .., t_n.\ (\wft_{\gamma_1}(t_1) \land .. \land \wft_{\gamma_n}(t_n)) \rightarrow  \truth(\bm p(t_1,.., t_n)) \leftrightarrow p(E_{\gamma_1}(t_1), .., E_{\gamma_n}(t_n)) \tag{A1\expofinite} \\
    & \forall t_1, t_2.\ (\reach_q(t_1) \land \reach_q(t_2))  \rightarrow  \truth(t_1 \bland t_2) \leftrightarrow (\truth(t_1) \land \truth(t_2)) \label{type:ax:T2} \tag{A2\expofinite} \\
    & \forall t_1.\ \reach_q(t_1) \rightarrow \truth(\bneg t_1) \leftrightarrow (\neg \truth(t_1)) \label{type:ax:T3} \tag{A3\expofinite} \\
    & \forall t_1.\ \reach_q(t_1) \rightarrow \label{type:ax:T4}\tag{A4\expofinite} \truth(\bforall(\fquo{x},t_1)) \leftrightarrow (\forall x^\gamma.\ \truth(\sub(t_1,\fquo{x},quote_\gamma(x))))
\end{align}
\textbf{Range:} $p \in P \setminus \{\truth\}, \gamma, \gamma_1, ..., \gamma_n \in U, \bm x \in \bm V \text{ of type } \gamma$

\begin{align}
    & \forall t_1, t_2^c.\ \reach(t_1) \rightarrow \textit{ist}(t_2^c, \bforall(\fquo{x},t_1)) \rightarrow  \forall x^\gamma.\ \textit{ist}(t_2^c, \sub(t_1,\fquo{x},\quot_\gamma(x))) \tag{A11\expofinite}\label{type:ax:istForallFinite}
\end{align}
\textbf{Range:} $\gamma \in U, \bm x \in \bm V \text{ of type } \gamma$

\subsection{Typing the Temporal Axioms}
Lastly, we type the axioms of temporal-Qiana (Section~\ref{sec:Time}) with the types introduced in Table~\ref{tab:types}.

\begin{align}
    &\forall f^q, t^\tau .\ \textit{HoldsAt}(f^q, t^\tau) \leftarrow \textit{Initially}_P(f^q) \land \neg \textit{Clipped}(0, f^q, t^\tau) \tag{EC1} \\
    &\forall f^q, t_3^\tau, a^a, t_1^\tau, t_2^\tau .\ \textit{HoldsAt}(f^q, t_3^\tau) \notag \\ & \hspace*{1cm} \leftarrow \textit{Happens}(a^a, t_1^\tau, t_2^\tau) \land \textit{Initiates}(a^a, f^q, t_1^\tau) \land \neg \textit{Clipped}(t_1^\tau, f^q, t_3^\tau) \land t_2^\tau < t_3^\tau \tag{EC2} \\
    &\forall t_1^\tau, f^q, t_4^\tau .\ \textit{Clipped}(t_1^\tau, f^q, t_4^\tau) \leftrightarrow \exists a^a, t_2^\tau, t_3^\tau .\ \textit{Happens}(a^a, t_2^\tau, t_3^\tau) \land \notag \\ &\hspace*{1cm} (\textit{Terminates}(a^a, f^q, t_2^\tau) \lor \textit{Releases}(a^a, f^q, t_2^\tau)) \land t_1^\tau < t_3^\tau \land t_2^\tau < t_4^\tau \tag{EC3} \\
    &\forall f^q, t^\tau .\ \neg \textit{HoldsAt}(f^q, t^\tau) \leftarrow \textit{Initially}(\bneg f^q) \land \neg \textit{Declipped}(0, f^q, t^\tau)\tag{EC4} \\
    &\forall f^q, t_3^\tau, a^a, t_1^\tau, t_2^\tau .\ \neg \textit{HoldsAt}(f^q, t_3^\tau) \notag \\ & \hspace*{0.5cm} \leftarrow \textit{Happens}(a^a, t_1^\tau, t_2^\tau) \land \textit{Terminates}(a^a, f^q, t_1^\tau) \land \neg \textit{Declipped}(t_1^\tau, f^q, t_3^\tau) \land t_2^\tau < t_3^\tau \tag{EC5} \\
    &\forall t_1^\tau, f^q, t_4^\tau .\ \textit{Declipped}(t_1^\tau, f^q, t_4^\tau) \leftrightarrow \exists a^a, t_2^\tau, t_3^\tau .\ \tag{EC6} \\ &\hspace*{1cm} \textit{Happens}(a^a, t_2^\tau, t_3^\tau) \land (\textit{Initiates}(a^a, f^q, t_2^\tau) \lor \textit{Releases}(a^a, f^q, t_2^\tau)) \land t_1^\tau < t_3^\tau \land t_2^\tau < t_4^\tau \notag \\
    &\forall a^a, t_1^\tau, t_2^\tau .\ \textit{Happens}(a^a, t_1^\tau, t_2^\tau) \rightarrow t_1^\tau \leq t_2^\tau\tag{EC7}
\end{align}

\section{Qiana as a Modal Logic} \label{sec:OtherFormalisms}
    We now show that propositional modal logic can be represented in Qiana. 

\subsection{Background on Modal Logic}

For our purposes, the set of modal formulas is defined by adding the arity one ``necessity'' operator  $\Box$ to the inductive definition of propositional formulas.
The counterpart ``possibility'' operator $\Diamond$ is then defined by $\Diamond \varphi \Leftrightarrow \neg \Box \neg \varphi$.
We recall the usual axioms of formal logics and the definitions of the associated systems ($K$, $T$, $S4$, $S5$, $D$)~\cite{Chellas1980ModalLA,BlackburnModal,stanford-logic-modal}.
These are the axioms:
\begin{align*}
     K:\ \ & \Box (p \rightarrow q) \rightarrow (\Box p \rightarrow \Box q) \\
     T:\ \ & \Box p \rightarrow p \\
     4:\ \ & \Box p \rightarrow \Box \Box p \\
     5:\ \ & \Diamond p \rightarrow \Box \Diamond p \\
     D:\ \ & \Box p \rightarrow \Diamond p
\end{align*}

There is also a rule called \emph{necessitation}, which states that if a formula is a tautology of the system $\textbf{K}$ (see below), then it can be inferred from the axioms:
\[N: \frac{\textbf{K} \models p}{\Box p}\]

\noindent These are the systems:
\begin{align*}
     \textbf{K} &\coloneq K + N \\
     \textbf{T} &\coloneq \textbf{K} + T \\
     \textbf{S4} &\coloneq \textbf{T} + 4 \\
     \textbf{S5} &\coloneq \textbf{T} + 5 \\
     \textbf{D} &\coloneq \textbf{K} + D
\end{align*}

\subsection{Translation of Standard Modal Logic into Qiana}

We show how to translate the usual propositional modal logics $\textbf{K}$, $\textbf{T}$, $\textbf{S4}$, $\textbf{S5}$, and $\textbf{D}$ into Qiana.
For the purpose of this section, we assume an augmented signature with no functions or predicates except for arity 0 predicates (which we take as our propositional variables) and the various predicates and functions necessary for the definition of Qiana or introduced in this section.

We introduce a special 0-arity function symbol $\Box$ that we use as a context for the usual propositional modality $\Box$.
We introduce the following notations:
\[\Box \bm\varphi \coloneq \ist(\Box, \fquo{\varphi})\]
\[\Diamond \bm\varphi \coloneq \neg \ist(\Box, \bneg \fquo{\varphi})\]

Thanks to these notations, we can interpret a modal formula as a Qiana formula.
To adapt the axioms $K$, $T$, 4, 5, and $D$ to Qiana, we need to explicitly quantify on quoted formulas and to replace any logical connective with its quotation. 
Because $\Box$ is a context (and $\Diamond$ is only defined as an abbreviation using $\Box$), we add a level of quotation per level of nesting.
\begin{align*}
     QK:\ \ & \forall p, q.\ \Box (p \bm\rightarrow q) \rightarrow (\Box p \rightarrow \Box q) \\
     QT:\ \ & \forall p.\ \Box p \rightarrow \truth(p) \\
     Q4:\ \ & \forall p.\ \Box p \rightarrow \Box \fquo{\Box p}\\
     Q5:\ \ & \forall p.\ \Diamond p \rightarrow \Box \fquo{\Diamond p}\\
     QD:\ \ & \forall p.\ \Box p \rightarrow \Diamond p
\end{align*}

To adapt the necessitation rule $N$ to Qiana, we will need to add a new unary predicate $\tauto$.
Intuitively, $\tauto(\bm \varphi)$ holds if and only if $\textit{Qiana} \models \varphi$, with $\textit{Qiana}$ the set of Qiana axioms on the signature at hand. 
Using the predicate $\tauto$, we can write the necessitation rule as follows:
\[QN:\ \ \forall p.\ \tauto(p) \rightarrow \Box p\]

\noindent To constrain \tauto\ we introduce the predicate $\wff$ to represent well-formed formulas:
\begin{align*}
     & \wff(\bm p()) \tagtautodef \\
     & \forall A.\ \wff(A) \rightarrow \wff(\bneg A) \tagtautodef \\
     & \forall A, B.\ \wff(A) \land \wff(B) \rightarrow \wff(A \bland B) \tagtautodef \\
     & \forall A.\ \wff(A) \rightarrow \wff(\bm\ist(\bm\Box, \quot(A))) \tagtautodef
\end{align*}

\noindent\textbf{Range:} $p \in P \text{ with } |\delta(p)| = 0$ \\

\noindent We adapt a simple Hilbert-style proof system of propositional logic~\cite{mendelson2009introduction} to Qiana. 
\begin{align*}
     & A \rightarrow (B \rightarrow A) \\
     & (A \rightarrow (B \rightarrow C)) \rightarrow ((A \rightarrow B) \rightarrow (A \rightarrow C)) \\
     & (\neg B \rightarrow \neg A) \rightarrow ((\neg B \rightarrow A) \rightarrow B)
\end{align*}

\noindent This can be done with the following axioms:
\begin{align*}
     & \forall A, B.\ \wff(A) \land \wff(B) \rightarrow \tauto(A \bm\rightarrow (B \bm\rightarrow A)) \tagtautodef \\
     & \forall A, B, C.\ \wff(A) \land \wff(B) \land \wff(C) \rightarrow \tauto((A \bm\rightarrow (B \bm\rightarrow C)) \bm\rightarrow ((A \bm\rightarrow B) \bm\rightarrow (A \bm\rightarrow C))) \tagtautodef \\
     & \forall A, B.\ \wff(A) \land \wff(B) \rightarrow \tauto((\neg B \bm\rightarrow \neg A) \bm\rightarrow ((\neg B \bm\rightarrow A) \bm\rightarrow B)) \tagtautodef
\end{align*}

\noindent We also need to apply the \textit{modus ponens} rule to $\tauto$:
\[\forall A, B.\ \tauto(A \bm\rightarrow B) \land \tauto(A) \rightarrow \tauto(B) \tagtautodef\]

\noindent Lastly, there is a more complex issue to consider.
The necessitation rule $N$ reacts to what is tautological within the system $\textbf{K}$ and not only to what is generally tautological in classic propositional logic.
Hence, we need to include rules to extend $\tauto$ to tautologies of $\textbf{K}$ rather than only tautologies of classical propositional logic.
\begin{align}
     & \forall A.\ \tauto(A) \rightarrow \tauto(\bm\ist(\bm\Box,\quot(A))) \tagtautodef \\
     & \forall A, B.\ \tauto(\bm\ist(\bm\Box, A \bm\rightarrow B)) \rightarrow \tauto(\bm\ist(\bm\Box, A) \bm\rightarrow \bm\ist(\bm\Box, B)) \tagtautodef \label{eq:tautoLast}
\end{align}

\noindent Together, these axioms define the behavior of $\tauto$.

\begin{definition}
     We define $H_\text{tauto} = T1-\ref{eq:tautoLast}$, the set of axioms that define the behavior of $\tauto$.
\end{definition}

\begin{lemma}[Definition of $\tauto$] \label{lemma:tauto}
     If $\varphi$ is a propositional logic formula and $\textbf{K} \models \varphi$, then $H_\text{tauto} \models \tauto(\bm\varphi)$.
\end{lemma}
\begin{proof}
     Any tautology of $\textbf{K}$ can be proven by a finite number of applications of the axioms of $\textbf{K}$ and the rules \textit{modus ponens} and necessitation.
     These admit counterparts under $H_\text{tauto}$, which means each step of such a proof is a valid inference under $H_\text{tauto}$.
\end{proof}

We now define sets of Qiana axioms corresponding to the various systems of modal logic:
Let $H_\text{Q}$ be the Qiana axioms on the signature at hand.
Then, we can define the following systems:
\begin{align*}
     H_\textbf{QK} &= H_\text{Q} + H_\text{tauto} + QK + QN \\
     H_\textbf{QT} &= H_\textbf{QK} + QT \\
     H_\textbf{Q4} &= H_\textbf{QT} + Q4 \\
     H_\textbf{Q5} &= H_\textbf{QT} + Q5 \\
     H_\textbf{QD} &= H_\textbf{QK} + QD
\end{align*}

\noindent We prove in Proposition~\ref{prop:modal} that these systems are equivalent to the usual systems of modal logic.
\begin{proposition} \label{prop:modal}
     Let $\varphi$ be a propositional modal formula. 
     Then :
     \begin{align*}
           H_\textbf{QK} &\models \varphi \text{ if and only if }\ \textbf{K} \models \varphi \\
           H_\textbf{QT} &\models \varphi \text{ if and only if }\ \textbf{T} \models \varphi \\
           H_\textbf{Q4} &\models \varphi \text{ if and only if }\ \textbf{S4} \models \varphi \\
           H_\textbf{Q5} &\models \varphi \text{ if and only if }\ \textbf{S5} \models \varphi \\
           H_\textbf{QD} &\models \varphi \text{ if and only if }\ \textbf{D} \models \varphi
     \end{align*}
\end{proposition}
\begin{proof}
     We describe the proof that $H_\textbf{QK} \models \varphi \text{ if and only if }\ \textbf{K} \models \varphi$. The other cases are similar.\\

     \noindent We first prove that if $\textbf{K} \cup \Gamma \models \bot$ for some set $\Gamma$ of modal formulas then $H_\textbf{QK} \cup \Gamma \models \bot$.

     If $\textbf{K} \cup \Gamma \models \bot$, then there is some proof of $\bot$ from $\Gamma$ using only the axiom $K$, the axioms of classical propositional logic listed above, and the rules \textit{modus ponens} and necessitation.
     Each of these axioms and rules has a counterpart in the formulas of $H_\textbf{QK}$.
     The only nontrivial case is the necessitation rule $N$, which relies on $\tauto$, the behavior of which was shown in Lemma~\ref{lemma:tauto}.
     Hence, each step of the proof is a valid inference under $H_\textbf{QK}$.
     Therefore, $H_\textbf{QK} \cup \Gamma \models \bot$. \\

     \noindent Now we prove the other direction; if $\textbf{K} \cup \Gamma \not\models \bot$ then $H_\textbf{QK} \cup \Gamma \not\models \bot$.
     Assume $\textbf{K} \cup \Gamma \not\models \bot$.
     Then there is a model $M$ of $\textbf{K} \cup \Gamma$.
     We can adapt $M$ to a model $M_Q$ of $H_\textbf{QK} \cup \Gamma$. We briefly describe $M_Q$:
     \begin{itemize}
          \item The truth value of every arity 0 predicate is the same as in $M$.
          \item For every formula $\varphi$ of modal logic, we have $M_Q \models \Box \varphi$ if and only if $M \models \Box \varphi$. This is coherent with all axioms of $H_\textbf{QK}$.
          \item $\wff$ is true on every quotation of a modal formula and false otherwise. This is coherent with the axioms of $H_\textbf{QK}$.
          \item $\tauto$ is true on every tautology of $\textbf{K}$ and false otherwise. This is coherent with the axioms of $H_\textbf{QK}$.
     \end{itemize}
     The model $M_Q$ behaves like $M$ on all modal formulas and is therefore a model of $\Gamma$. Because we can check it is a model of $H_\textbf{QK}$, we have $H_\textbf{QK} \cup \Gamma \not\models \bot$. \\

     \noindent We have shown that $H_\textbf{QK} \models \varphi \text{ if and only if }\ \textbf{K} \models \varphi$ by proving both directions of the equivalence.
\end{proof}


\subsection{Example}

We show an example of a simple line of reasoning in the system \textbf{D} of modal logic and its translation to Qiana.
In natural language, we take the following premises:
\begin{itemize}
    \item ``Necessarily, either Juliet faked her death or she killed herself.''
    \item ``Necessarily, if Juliet faked her death, she was unhappy with her family's decision.''
    \item ``Necessarily, if Juliet killed herself, she was unhappy with her family's decision.''
\end{itemize}

\noindent We want to prove that it is possible that Juliet was unhappy with her family's decision. \\

\noindent \textbf{First, let us do this proof in classical modal logic:}

\begin{align*}
    A &: \text{Juliet faked her death} \\
    B &: \text{Juliet killed herself} \\
    C &: \text{Juliet was unhappy with her family's decision}
\end{align*}

\noindent The premises can be formalized as:
\begin{align*}
    & \Box (A \lor B) \\
    & \Box (A \rightarrow C) \\
    & \Box (B \rightarrow C)
\end{align*}

\noindent We notice the following propositional tautology:
\begin{equation*}
    \models (A\rightarrow C) \rightarrow ((B \rightarrow C) \rightarrow ((A \lor B) \rightarrow C))
\end{equation*}

\noindent By necessitation ($N$), this gives:
\begin{equation*}
    \Box (A\rightarrow C) \rightarrow ((B \rightarrow C) \rightarrow ((A \lor B) \rightarrow C))
\end{equation*}

\noindent By applying the \textit{modus ponens} rule multiple times with $K$ we derive:
\begin{equation*}
    \Box C
\end{equation*}

\noindent Then, axiom $D$ gives us:
\begin{equation*}
    \Diamond C
\end{equation*}

\noindent \textbf{We now translate this proof to Qiana.} We define arity-0 predicates $A$, $B$, and $C$ to represent the propositions. The necessity operator ($\Box$) is treated as a context. The translation of the premises is:

\begin{align*}
    & \ist(\Box, \fquo{A \lor B}) \\
    & \ist(\Box, \fquo{A \rightarrow C}) \\
    & \ist(\Box, \fquo{B \rightarrow C})
\end{align*}

\noindent Which can be written more concisely as:

\begin{align*}
    & \Box \fquo{(A \lor B)} \\
    & \Box \fquo{A \rightarrow C} \\
    & \Box \fquo{B \rightarrow C}
\end{align*}

\noindent We derive $\tauto(\fquo{(A\rightarrow C) \rightarrow ((B \rightarrow C) \rightarrow ((A \lor B) \rightarrow C))})$ by Lemma~\ref{lemma:tauto}. Then, by the rule \textit{Modus Ponens} and Axiom $QN$, we obtain:

\begin{equation*}
    \Box \fquo{(A\rightarrow C) \rightarrow ((B \rightarrow C) \rightarrow ((A \lor B) \rightarrow C))}
\end{equation*}

\noindent By repeatedly applying $QK$ and the rule \textit{Modus Ponens}, we derive:

\begin{equation*}
    \Box C
\end{equation*}

\noindent Finally, using axiom $QD$ and the rule \textit{Modus Ponens}, we conclude:

\begin{equation*}
    \Diamond C
\end{equation*}

\noindent As we can see, the structure of the proof is very similar to that of pure modal logic. The only steps that do not mirror the ones from the modal logic reasoning are those where we use the $\tauto$ predicate, but its behavior is guaranteed by Lemma~\ref{lemma:tauto}.

 \section{Additional Remarks} \label{sec:discussion}

In this section, we present some remarks and discussions regarding the use of Qiana.
\SC{JAIR does not want us to start a section with a subsection, so I added asomething.}

\subsection{About the Definition of the Set $V$ of Quotable Variables}

In Section~\ref{sec:FOdefs} we defined $V$ as a finite subset of $V_\infty$, which is in bijection with $\cv$. 
In formulas such as Axiom~\ref{ax:str4}, we connect $x$ (an element of $V$) with $\bm x$ (the corresponding element of $\cv$).
This is a convenient way to present our axioms without a lengthy discussion on fresh variables and the like.
However, in first-order logic, bound variables can be freely renamed with fresh variables. In some schemes (like Formula~\ref{ax:str4}), we limited the range of some $x$ to $V$ rather than $V_\infty$. But all the bound variables can be replaced with elements of $V_\infty$ without issue. 
The only aspect of $V$ that matters is that it has the same size as $\cv$. The notion of ``quotable variable'' amounts to a limitation on the number of distinct variables in a quotable formula, along with giving us a convenient way to state our axioms.

\subsection{Useful Notation for Handwritten Formulas}

When writing formulas by hand and nesting contexts, it can be useful to have a notation that can be joined with itself by simple concatenation.
For example, consider without additional notation the formula for ``Romeo believes that Juliet believes that Romeo believes he is pretty'' (ie, Romeo is aware Juliet considers him vain).

\[\ist(\textit{Romeo}, \bmnested{\ist(\textit{Juliet}, \bmnested{\ist(\textit{Romeo}, \bmnested{\textit{Pretty}(\textit{Romeo})})})})\]

\noindent Since this can become cumbersome to write, we propose the following notation:
\[\context{c} t\hspace*{0.3cm} := \ist(c, t)\]
\[\contextq{c} \varphi := \ist(c, \fquo{\varphi})\]

\noindent With this notation, the formula above becomes:
\[\contextq{\textit{Romeo}} \contextq{\textit{Juliet}} \contextq{\textit{Romeo}} \textit{Pretty}(\textit{Romeo})\]

\noindent Which is easier to read and to write.

\subsection{Axioms for Disambiguation}

The axioms on the behavior of contexts introduced in Section~\ref{sec:axioms} are minimal.
In particular, we do not enforce the correct interpretation of terms in a context. For example, the formula $\ist(c, \bm P(\bm 2))$ does not imply $\ist(c, \bm P(\fquo{1+1}))$; even where 1+1=2 is taken for granted.
This nonenforcement of the correct interpretation of terms in contexts extends to the function $\quot$. 
The recursive definition of $\truth$ includes a mechanism to unwrap $\quot$ symbols, but this is not enforced for arbitrary contexts.

In applications where the correct interpretation of terms in contexts is important, the following two axioms can be included to enforce this behavior:
\[\forall c,t. \wft(t) \rightarrow \ist(c,t\ \bm=\ \quot(E(t)))\]
\[\forall c, t_1, t_2, t.\ \ist(c, t_1\ \bm=\ t_2) \rightarrow \ist(c, t \bm\leftrightarrow \sub(t,t_1,t_2))\]

\noindent
We recall that $\sub$ is substitution $t[t_1 \leftarrow t_2]_q = \sub(t,t_1,t_2)$; and $E$ is the evaluation symbol, matching a quoted term to its value (the opposite of $\quot$).
See Section~\ref{sec:FiniteAxio} for more on the definition of these symbols within Qiana.
 
\section{Conclusion} \label{sec:Conclusion}
    We have introduced Qiana, a formalism based on first-order logic that allows reasoning on contexts, quantifying over contexts, and quantifying over formulas. 
Thanks to our finite axiomatization process, Qiana theories can be used with any
TPTP-compatible theorem prover.
We have shown that Qiana can be used to model beliefs, stories, and paraconsistency.
Furthermore, we have extended Qiana to reason about temporality with 
event calculus. We have also presented an alternative many-sorted version of Qiana that includes time.
Finally, we have shown how the usual systems of modal logic can be written within Qiana.

We expect Qiana to be usable for and adaptable to various types of contextual reasoning cases, including reasoning on hypothetical scenarios, fake news, legal reasoning, and different points of view. 
In future work, we intend to produce the tools necessary to translate natural language knowledge to Qiana and to develop a Qiana-based system for reasoning on contexts and beliefs.
Once this is done, the ability of Qiana to quantify over both formulas and contexts while remaining compatible with automated theorem provers will make it a powerful tool for interpretability.

\subsection*{Acknowledgments} 
This work was partially funded by the NoRDF project (ANR-20- CHIA-0012-01).

\printbibliography

\appendix
\section{Proof of Truth Definition}
     \label{sec:proveTruth}

To prove Property~\ref{prop:truthSubsumption}, we will instead prove Lemma~\ref{lemma:addTruth}, which is stronger.
\begin{lemma} \label{lemma:addTruth}
    Let $A \in \tqfv$ with no free quoted variables (ie, each $\bm x$ is quantified by a $\bforall$).
    Let $x_1, \dots, x_n$ be the free variables of $A$.
    Then
    \[\AoneAfour \models \forall x_1, \dots, x_n.\ \truth(A) \leftrightarrow \uq(A)\]
\end{lemma}

We prove Lemma~\ref{lemma:addTruth} by induction on $A$.

\fbox{Base case}
Let $M$ be a model of \AoneAfour. Let us prove that $M \models \forall x_1, \dots, x_n.\ \truth(\bm p(t_1, \dots,t_m)) \leftrightarrow \uq(\bm p(t_1, \dots,t_m))$.
For all assignments $\assignment$ of variables $x_1, \dots, x_n$, we have:
    \begin{align*}
    & M, \assignment \models \truth(\bm p(t_1, \dots,t_m)) \\
    & \textiff M, \assignment \models p(\uq(t_1), \dots, \uq(t_m)) & \text{  as $M \models$ \ref{ax:str1}}\\
    & \textiff M, \assignment \models \uq(\bm p(t_1, \dots,t_m)) & \text{ by definition of $\uq$}
    \end{align*}
\fbox{Negation}
Let $M$ be a model of \AoneAfour. 
Let us prove that $M \models \forall x_1, \dots, x_n.\ \truth(\bneg A) \leftrightarrow \uq(\bneg A)$.
For all assignments $\assignment$ of variables $x_1, \dots, x_n$, we have:
    \begin{align*}
     &   M, \assignment \models \truth(\bneg A) \\
      &  \textiff M, \assignment \not \models \truth(A) & \text{ as $M \models$ \ref{ax:str3}}\\
       & \textiff M, \assignment \not \models \uq(A) & \text{ by IH} \\
        & \textiff M, \assignment \models \uq(\bneg A) & \text{ by definition of $\uq$}
    \end{align*}

\fbox{$\forall$}
Let $M$ be a model of \AoneAfour. Let us prove that $M \models\forall x_1, \dots, x_n.\ \truth(\bforall(\fquo{x} ,A)) \leftrightarrow \uq(\bforall(\fquo{x} ,A))$.
For all assignments $\assignment$ of variables $x_1, \dots, x_n$, we have:
    \begin{align*}
    &   M, \assignment \models \truth(\bforall(\fquo{x} ,A)) \\
    & \textiff M, \assignment \models \forall x.\ \truth(A[\fquo{x} \leftarrow quote(x)]_q) & \text{ as $M \models $ \ref{ax:str4}}\\
    & \textiff M, \assignment \models \forall x. \uq(A[\fquo{x} \leftarrow quote(x)]_q) & \text{ by IH} \\
    & \textiff M, \assignment \models  \uq(\bforall(\fquo{x} ,\bneg A)) & \text{ by definition of $\uq$}
    \end{align*}

\section{Proof of Finite Axiomatization}
     \label{sec:proveFiniteAxio}
     We will now prove Theorem~\ref{th:mainThProveFiniteAxio}.
To simplify the proof, we will omit the existence of schema~\ref{as:oldDefss} and its finite counterpart schema~\ref{ax:istForallFinite}.
The reason is that they are vastly orthogonal to the other difficulties of the proof and can be handled in the same way as we deal with the other axioms in this proof, except they form a more straightforward case.
Hence, they would only bloat the proof with unnecessary tedium, largely redundant in spirit with the rest of the reasoning. \\

\noindent Therefore, for the purpose of this proof, we assume:
\begin{align*}
    \hq & := H \cup \hqf \cup \htr \\
    \hf & := H \cup \hqf \cup \htf \cup \hof
\end{align*}

We now prove that $\hq$ is coherent if and only if $\hf$ is coherent through Proposition~\ref{prop:fiFinGiveBase} and Proposition~\ref{prop:fiBaseFiveFin}.

\begin{proposition} \label{prop:fiFinGiveBase}
    \[\hf \not\models \bot \rightarrow \hq \not\models \bot\]
\end{proposition}
\begin{proof}
    We want to prove that if $\hf = H \cup \hqf \cup \hof \cup \htf$ has a model, then $\hq = H \cup \hqf \cup \htr$ also has one. 
    We will now prove $\hf \models \htr$, which is sufficient. \\

    \begin{lemma}
        Let $t_1, t_2 \in \tqv$ with free variables $x_1, \dots, x_m$, all within quotes. \\
        \hspace*{2cm} Then for all $\fquo{x}$, $\hf \models \forall x_1, ..., x_m.\ \sub(t_1, \fquo{x}, t_2) = t_1[\fquo{x} \leftarrow t_2]_q$.
    \end{lemma}
    \begin{proof}
        Proven by direct recursion on $t_1$.
    \end{proof}

    \begin{lemma}
        Let $t \in \tqv$ with free variables $x_1, \dots, x_m$, all within quotes. Then $\hf \models \forall x_1, \dots, x_m.\ \reach(t)$.
    \end{lemma}
    \begin{proof}
        Proven by direct recursion on \truth.
    \end{proof}

    \begin{lemma} \label{lemma:EisXi}
        Let $t \in \tqtv$ with free variables $x_1, \dots, x_m$, all within quotes. Then
        \[\hf \models \forall x_1, \dots, x_n.\ E(t) = \uq(t)\]
    \end{lemma}
    \begin{proof}
        Proven by direct recursion on \truth, as $E$ is by construction built on the same recursion as $\uq$.
    \end{proof}
    Armed with these lemmas, we can prove all schemas of $\htr$ with their direct counterpart from $\htf$. \\

    We provide a sketch for the more complicated case of schema~\ref{ax:str1}, highlighting the most important elements of the proof. \\
    Let $t_1,\dots, t_k \in \tqtv$ with free variables $x_1,\dots,x_n$. \\
    By construction of $\wft$ we have $\forall x_1,\dots, \forall x_n.\ \wft(t_i)$ for all $i$.
    Hence by schema~\ref{ax:T1} we have $\forall x_1,\dots, \forall x_n.$ $\ T(\bm p(t_1,\dots, t_k)) \leftrightarrow p(E(t_1), \dots, E(t_k))$. 
    Lemma~\ref{lemma:EisXi} allows us to conclude.
\end{proof}

\noindent We now prove the other direction of the equivalence.
\begin{proposition} \label{prop:fiBaseFiveFin}
    \[\hq \not\models \bot \rightarrow \hf \not\models \bot\]
\end{proposition}
\begin{proof}
    Let $M$ be a model of $\hq$. We will define a model $M_f$ and prove that it is a model of $\hf$. We define $M_f$ as follows: \\

    Let $D$ be the domain of $M$. Recall that $\termset$ is the set of all terms under $S$. Without loss of generality, we assume $D \cap \termset = \emptyset$.
    We define $D_f$ as the set obtained by adding $D$ and removing variables to and from the recursive definition of $\termset$. 
    We define the $M$-interpretation of elements of $t \in D_f$ as the value in $D$ that we obtain by recursively evaluating \truth under $M$.
    \begin{itemize}
        \item Under $M_f$, the interpretation of any function symbol from $S$ is to recursively build the element of $D_f$. Intuitively, we ``only'' store the terms as we evaluate them without performing any other operation.
        \item Under $M_f$, the interpretation of any predicate symbol is to turn all arguments in $D_f$ to their $M$-interpretation and then interpret the predicate as in $M$.
        \item $=$ is true equality on $D_f$.
        \item $Wft$ is true only on elements of $D_f$ recursively built with elements of $\cv$, $\fq$, and $quote(y)$ for $y \in D$. Remark that $\wft$ is the minimal predicate that satisfies schemas~\ref{wft1} to \ref{wft4} and is built with schemas that follow the recursive construction of $\tqtv$.
        \item $\reach$ is likewise the minimal predicate to satisfy its definition schemas (schemas~\ref{reach1} to \ref{reach3}). It is only true on elements recursively built as terms.
        \item $\sub(t_2, \fquo{x}, t_1)$ is recursive and localy behaves as $t_2[\fquo{x} \leftarrow t_1]_q$ if the top symbol of $t_2$ is coherent with $t_2$ being in $\tqv$. Otherwise $\sub$ simply returns $t_2$.
        If the first argument of $\sub$ is not in $\cv$, then $\sub$ returns $t_2$. 
        The intuition is that $\sub$ is a recursive function defined as ``behaves like $t_2[\fquo{x} \leftarrow t_1]_q$ if it makes sense to do so (locally). Otherwise, return $t_2$.''.
        \item $E$ likewise behaves like $\uq$ where it locally makes sense to do so and otherwise is identity. Remark that $E(t) = \uq(t)$ for $t \in \tqtv$.
    \end{itemize}

    \begin{lemma} \label{lemma:MfIsOftenM}
        Let $\varphi$ be a formula well-defined on $S$.
        Then $M \models \varphi$ iff $M_f \models \varphi$ 
    \end{lemma}
    \begin{proof}
       (Sketch) At each step of term evaluation, the $M$-interpretation of the result is equal to the same operation applied to the $M$ interpretations of the arguments. 
       Since every term will have to be interpreted through a predicate symbol from $S$, and therefore sent to its $M$-interpretation before evaluation, everything happens as though only the $M$-interpretation of the values was considered.
       This is exactly the application of $M$ itself to the formulas. 
       This proves Lemma~\ref{lemma:MfIsOftenM}.
    \end{proof}
    We recall that $M \models \hq$, which means $M \models H \cup \hqf \cup \htr \cup \hts$.
    Hence and by Lemma~\ref{lemma:MfIsOftenM}, $M_f \models H$ and $M_f \models \hqf$. 
    By directly applying the definitions, we see that $M_f \models \hof$. 
    Since $\hf = H \cup \hqf \cup \hof \cup \htf$, we now need to prove only that $M_f \models \htf$. 
    We prove this by checking that $M_f$ accepts schemas~\ref{ax:T1} to \ref{ax:T4}.
    As they are quite similar to one another, we only provide explanations for the more complicated case of schema~\ref{ax:T1}.

    Let $p$ be a predicate symbol of arity $n$. Further, let $v_1, \dots, v_n \in D_f$ such that $M_f \models \wft(v_i)$ for all $i$.
    By definition of $\wft$, there is some selection of variables $x_1, \dots, x_m$, a valuation $\sigma: x_1, \dots, x_m \rightarrow D_f$, and terms $t_1, \dots, t_n \in \tqtv$ such that: $ \forall i \in [1,n], (M_f, \sigma)(t_i) = v_i$.
    By definition of $\htr$ and Lemma~\ref{lemma:MfIsOftenM}, we have:
    $M_f \models \forall x_1, \dots, x_m.\ T(\bm p(t_1, \dots, t_n)) \leftrightarrow p(\uq(t_1), ..., \uq(t_n))$. 
    Hence $M_f, \sigma \models  T(\bm p(t_1, \dots, t_n)) \leftrightarrow p(\uq(t_1), ..., \uq(t_n))$. 
    $E$ equals $\uq$ on $\tqtv$, hence this gives $M_f, \sigma \models  T(\bm p(t_1, \dots, t_n)) \leftrightarrow p(E(t_1), ..., E(t_n))$.  
    Finally, this gives $M_f \models T(p(v_1, \dots, v_n)) \leftrightarrow p(E(v_1), \dots, E(v_n))$, which is what we wanted to prove.

    Remark that we can handle the case of schema~\ref{ax:T4} in a vastly similar fashion, relying on the similarity of $\sub$ to $\_[\_ \leftarrow \_]_q$ instead of the similarity of $E$ to $\uq$.
\end{proof}

\vspace*{2cm}
\noindent Received 20 February 2025; accepted 29 September 2025.

\end{document}